\documentclass[aos]{imsart}

\RequirePackage{amsthm,amsmath,amsfonts,amssymb}
\RequirePackage[numbers,sort&compress]{natbib}
\RequirePackage[colorlinks,citecolor=blue,urlcolor=blue]{hyperref}
\RequirePackage{graphicx}

\startlocaldefs
\theoremstyle{plain}

\newtheorem{remark}{Remark}
\newtheorem{theorem}{Theorem}[section]
\newtheorem{lemma}[theorem]{Lemma}
\newtheorem{corollary}[theorem]{Corollary}
\newtheorem{proposition}[theorem]{Proposition}
\newtheorem{assumption}[theorem]{Assumption}

\theoremstyle{definition}
\newtheorem{definition}[theorem]{Definition}

\usepackage{tikz}
\usepackage{pgfplots}
\pgfplotsset{compat=1.16}
\usepgfplotslibrary{fillbetween}
\makeatletter
\def\journal@name{}
\def\journal@url{}
\makeatother

\endlocaldefs

\begin{document}

\begin{frontmatter}
\title{The Condition-Number Principle for Prototype Clustering}
\runtitle{The Condition-Number Principle for Prototype Clustering}
\begin{aug}

\author[A]{\fnms{Romano}~\snm{Li}\ead[label=e1]{romano.li@duke.edu}},
\author[B]{\fnms{Jianfei}~\snm{Cao}\ead[label=e2]{j.cao@northeastern.edu}}
\address[A]{Department of Economics, Duke University\printead[presep={,\ }]{e1}}

\address[B]{Department of Economics,
Northeastern University\printead[presep={,\ }]{e2}}
\end{aug}

\begin{abstract}
We develop a geometric framework that links objective accuracy to structural recovery in prototype-based clustering. The analysis is algorithm-agnostic and applies to a broad class of admissible loss functions.
We define a clustering condition number that compares within-cluster scale to the minimum loss increase required to move a point across a cluster boundary. When this quantity is small, any solution with a small suboptimality gap must also have a small misclassification error relative to a benchmark partition. The framework also clarifies a fundamental trade-off between robustness and sensitivity to cluster imbalance, leading to sharp phase transitions for exact recovery under different objectives. The guarantees are deterministic and non-asymptotic, and they separate the role of algorithmic accuracy from the intrinsic geometric difficulty of the instance.
We further show that errors concentrate near cluster boundaries and that sufficiently deep cluster cores are recovered exactly under strengthened local margins. Together, these results provide a geometric principle for interpreting low objective values as reliable evidence of meaningful clustering structure.
\end{abstract}



\end{frontmatter}

\section{Introduction}

\subsection{The disconnect between optimization and clustering consistency}\label{subsec:1.1}

Prototype-based clustering methods represent each group by a prototype and assign observations to the nearest prototype \cite{jain_algorithms_1988}. Classical examples include $k$-means and $k$-medoids \cite{mcqueen_methods_1967,kaufman_finding_2009}. These methods are widely used in statistics and machine learning because they yield simple, interpretable partitions. They are typically formulated as optimization problems over partitions and prototypes, but the resulting objectives are nonconvex and are therefore solved only approximately in practice, often via heuristics or relaxations. As a result, even when an algorithm achieves a low objective value, the scientific target remains the induced partition rather than the loss itself.
 
This raises the question of when near-optimality of the clustering objective implies that the resulting partition is structurally correct. In many applications, such as inferring heterogeneous treatment effects in econometrics or identifying cell types in genomics, the primary goal is to recover meaningful groups rather than to minimize a loss function $\mathcal{L}_n$. In such settings the relevant notion of success is structural accuracy.

Optimization success does not, by itself, guarantee structural recovery. A candidate solution $(\hat{\mathcal{C}},\hat{\theta})$ may achieve an objective value extremely close to the global minimum (a small \emph{objective gap} $\delta$), yet induce a partition that differs substantially from the desired benchmark (a large \emph{misclassification rate} $p$, defined up to label permutation). This phenomenon arises when the loss landscape is relatively flat along directions that alter the partition, so markedly different clusterings can produce nearly indistinguishable objective values. In such regimes, even an exact global minimizer may fail to recover the benchmark structure if the geometry is obscured by heavy tails, outliers, or severe imbalance.

Existing theoretical work often addresses this issue indirectly. Statistical analyses typically impose strong distributional assumptions, such as well-separated mixture models, while algorithmic analyses study the behavior of specific optimization procedures, such as Lloyd-type updates under favorable initialization. While these approaches provide valuable insights, they do not directly answer a more basic question:

Given a candidate solution that is nearly optimal in objective value, under what geometric conditions must it also be structurally close to a benchmark partition?

Answering this question would provide a direct bridge between optimization accuracy and structural inference. In particular, it would yield a certificate of correctness that applies independently of how the solution was obtained, whether through heuristics, relaxations, or exact optimization.

\subsection{Our contribution: A geometric condition number}\label{subsec:1.2}

Our main result establishes a stability principle linking optimization accuracy to structural recovery in clustering. Specifically, we show that when the geometry of the instance is well-conditioned, any solution that is nearly optimal in objective value must also be close to a benchmark partition.

The key observation underlying this result is geometric. Assigning a point to an incorrect prototype requires crossing a margin between clusters, which necessarily increases the loss. This allows us to lower bound the excess objective value contributed by each misclassified point. Aggregating these penalties yields a non-asymptotic inequality that connects the optimization gap to the misclassification rate.

This perspective leads to our main contribution: a geometric framework that links near-optimality to structural recovery.
We summarize the contribution in four parts.

\noindent\emph{1. The clustering condition number.}
Our analysis is organized around a dimensionless quantity that we call the \emph{clustering condition number} $\kappa$. This quantity measures the relationship between the scale of within-cluster variation and the minimum loss increment required to move a point across a cluster boundary under a given loss function $g$. In our benchmark geometry,
\[
\kappa \;\asymp\; \frac{g(D_{\text{eff}})}{\Delta_g(\gamma;D_{\text{eff}})},
\]
where $D_{\text{eff}}$ denotes an effective within-cluster radius, $\gamma$ denotes the benchmark geometric margin, and $\Delta_g(\gamma;D)$ denotes the minimal loss increment incurred by crossing a margin $\gamma$ starting from radius at most $D$.
Our main theorem (Theorem~\ref{thm:main_stability}) shows that the misclassification rate is controlled by the optimization gap through the condition number:
\[
p(\hat{\mathcal{C}},\mathcal{C}^*) \;\lesssim\;
\kappa \cdot (\delta+\delta_{\text{approx}})
\;+\; \text{(prototype-displacement term)}.
\]
The leading term is algorithm-agnostic. It separates the geometric difficulty of the instance, captured by $\kappa$, from the computational precision of the optimization procedure, captured by $\delta$. For standard objectives, the prototype displacement can itself be controlled by the optimization gap (Section~\ref{subsec:5.3}), yielding effectively one-parameter control in the small-gap regime.

\noindent\emph{2. Sharp phase transitions and objective selection.}
Specializing the framework to common clustering objectives yields sharp recovery thresholds and clarifies how robustness and cluster imbalance influence recoverability. In a two-ball model we obtain exact-recovery conditions for global minimizers, revealing distinct phase transitions under different losses. These results illustrate a fundamental trade-off between robustness to heavy-tailed geometry and sensitivity to cluster imbalance.

\noindent\emph{3. Local geometry and zero-error cores.}
Beyond global error bounds, we analyze the spatial distribution of misclassification. We introduce a Core--Belt decomposition (Section~\ref{sec:5}), showing that stability is inherently non-uniform, since points deep inside clusters enjoy stronger geometric margins and can therefore be certified as exactly recovered even when the global solution is only near-optimal. Structural ambiguity is thus confined to a narrow boundary region between clusters.

\noindent\emph{4. Operationalization.}
Finally, we extend the framework to heterogeneous objectives, including weighted and instance-specific losses, as well as hierarchical and dynamic clustering settings. We also develop a data-driven diagnostic procedure (Section~\ref{subsec:6.4}) that combines observable geometric quantities with empirical optimization gaps to produce conservative certificates of structural stability.


\subsection{Related work}\label{subsec:1.3}

This paper contributes to several streams of literature. 
First, a large statistical literature studies clustering recovery under explicit generative assumptions, especially finite mixtures. Foundational identifiability results are due to \cite{teicher_identifiability_1963,teicher_identifiability_1967}, while classical work by \cite{pollard_strong_1981,pollard_central_1982} proves consistency of empirical $k$-means minimizers. 
Related analyses in empirical vector quantization further establish convergence-rate results for empirically optimal quantizers under additional regularity conditions \cite{antos_individual_2005}.
Consistency results have been established for related methods such as $k$-medoids \cite{jiang_consistency_2021}. 
Subsequent papers analyze component and assignment recovery under separation or moment conditions, including Gaussian-mixture learning and spectral-clustering consistency, with recent extensions to broader regimes \cite{jaffe_asymptotic_2025,klochkov_robust_2021,telgarsky_moment-based_2013,dasgupta_learning_1999,moitra_settling_2010,luxburg_consistency_2008, biau_performance_2008,laloe_l1-quantization_2010,thorpe_convergence_2015,lember_minimizing_2003}. 
More recently, a prominent line of work has characterized minimax rates and information-theoretic phase transitions for exact recovery in mixture models \cite{ndaoud_sharp_2022}. 
These results provide important statistical guarantees, but they primarily concern risk/objective convergence under explicit generative models rather than whether a near-optimal objective value implies structurally accurate partition recovery in a deterministic, instance-wise manner.
   
A second strand studies optimization algorithms for clustering, including Lloyd-type methods for $k$-means and related approximation analyses \cite{arthur_k-means_2006,balcan_discriminative_2008,balcan_approximate_2009,awasthi_stability_2010,balcan_clustering_2013}. This literature also gives recovery guarantees under favorable initialization or separation assumptions, and analyzes spectral approaches followed by $k$-means under suitable gap conditions \cite{peng_partitioning_2015,von_luxburg_tutorial_2007}. 
A parallel effort investigates convex relaxations of clustering, particularly semidefinite programming (SDP), to bypass the non-convexity of $k$-means \cite{peng_approximating_2007,mixon_clustering_2016,fei_hidden_nodate}. 
These analyses often provide guarantees for the optimizer of a specific relaxation (or for relax-and-round procedures) under probabilistic separation assumptions, but they are tied to particular algorithmic constructions or certificate arguments rather than giving a generic, loss-based certificate for arbitrary near-optimal solutions.

A third line of research studies clustering stability under geometric or perturbation-based assumptions. The clusterability literature gives separation and stability conditions under which meaningful clusters are recoverable, and related work analyzes robustness to small perturbations in data or objectives \cite{ben-david_sober_2006,ben-david_clustering_2018,balcan_discriminative_2008}. 
In the ``beyond worst-case'' literature, perturbation resilience frameworks study instances whose optimal clustering is invariant to multiplicative perturbations \cite{bilu_are_2012,balcan_clustering_nodate,angelidakis_algorithms_2017}. 
Relatedly, robust statistics develops trimming- and contamination-based robustification of prototype clustering, including trimmed $k$-means and its extensions \cite{cuesta-albertos_trimmed_1997,garcia-escudero_general_2008}. 
These results clarify when clustering instances are well-conditioned, but they mainly describe data geometry or optimizer behavior, not near-optimal solutions. Thus they leave open the question of whether a near-optimal objective value guarantees structural closeness to a benchmark partition.

Conceptually, our framework is reminiscent of margin-based analyses in statistical learning, where a margin condition controls how probability mass concentrates near decision boundaries \cite{mammen_smooth_nodate}. We adapt this intuition to the deterministic, unsupervised, and non-convex setting of prototype clustering by using a uniform loss-increment geometry to penalize boundary crossings.

Our results are also useful for post-clustering inference. Recent work in econometrics and statistics studies inference after data-driven grouping, including treatment-effect and panel settings \cite{leung_network_2023,cao_inference_2025} and selective inference for clustering procedures \cite{chen_selective_2023,gao_selective_2024,yun_selective_2023}. A central concern in this literature is that uncertainty quantification can be distorted when the clustering step is unstable. Our geometric stability conditions provide a sufficient basis for arguing that near-optimal solutions are structurally close to a benchmark partition, which supports more reliable downstream inference.

\subsection{Organization}
Section \ref{sec:2} introduces the formal setup, including the benchmark partition, geometric margin quantities, admissible loss class, and definitions of objective and structural errors.
Section \ref{sec:3} presents the main stability theorem, which links optimization accuracy to structural recovery through a geometric condition number, together with approximation and displacement terms.
Section \ref{sec:4} instantiates this bound for specific losses and analyzes exact-recovery phase transitions. Section \ref{sec:5} refines the global guarantee through local core-belt analysis and displacement control. 
Section \ref{sec_diagnostics} presents the diagnostic implications of the theory. 
Section \ref{sec:6} extends the framework to heterogeneous and dynamic settings.
Section \ref{sec:7} concludes with a discussion of interpretation, practice, and future directions.
Appendix A establishes the tightness of the phase-transition thresholds in the two-ball model. Appendix B contains proofs and auxiliary technical results for the refined stability analysis. Appendix C provides the deferred proofs for the heterogeneous and dynamic tracking extensions.

\section{Statistical framework and geometric setup}\label{sec:2}

In this section, we formalize the prototype-based clustering framework. We introduce a flexible class of loss functions that encompasses standard algorithms (e.g., $k$-means, $k$-medoids) and robust variants (e.g., Huber clustering). Crucially, we distinguish between the \emph{optimization objective} used by the algorithm and the \emph{geometric benchmark} used to evaluate structural recovery.

\subsection{Notation}

We use the following conventions throughout the paper. For a positive integer $m$, write $[m]:=\{1,\dots,m\}$. 
We use $|A|$ for cardinality of a finite set $A$. Candidate (algorithmic) objects are denoted with hats, such as $(\hat{\mathcal{C}},\hat{\theta})$, while benchmark/reference objects carry a superscript $*$, such as $(\mathcal{C}^*,\theta^*)$.
For cluster labels, $\Pi_k$ denotes the set of all permutations of $[k]$. When comparing two partitions, all label-dependent quantities are understood up to relabeling via $\Pi_k$. 
Given a metric $d$, we write $B(c,r):=\{x\in\mathcal{X}: d(x,c)\le r\}$ for a closed ball. We use $a\lesssim b$ to mean $a\le Cb$ for a universal constant $C>0$ independent of $(n,k)$ and the instance-specific geometry, unless explicitly stated otherwise.

\subsection{Admissible losses}\label{subsec:2.1}

Let $\mathcal{X}$ be a domain equipped with a metric $d(\cdot, \cdot)$ (typically Euclidean distance in $\mathbb{R}^d$). We consider a dataset $S = \{x_1, \dots, x_n\} \subset \mathcal{X}$. A clustering of $S$ into $k$ groups is defined by a partition $\mathcal{C}=\{C_1,\dots,C_k\}$ of $[n]$ and a set of prototypes
$\theta=(\theta_1,\dots,\theta_k)\in\Theta^k$, where $\Theta\subseteq\mathcal{X}$
for medoid-type formulations and $\Theta=\mathbb{R}^d$ for center-based formulations.

We consider clustering methods that minimize a cumulative loss 
$$\mathcal{L}_n(\mathcal{C}, \theta) := \sum_{j=1}^k \sum_{i\in C_j} g(d(x_i, \theta_j)),$$
where $g$ is a loss function. 
To accommodate non-smooth, robust, or saturated objectives, we consider a broad class of admissible loss functions.

\begin{definition}[Admissible Loss Function]\label{def:admissible_loss}
A function $g: [0, \infty) \to [0, \infty)$ is an \emph{admissible loss} if:
\begin{enumerate}
    \item[(i)] $g(0) = 0$ and $g$ is non-decreasing.
    \item[(ii)] $g$ is continuous on $[0, \infty)$.
\end{enumerate}
\end{definition}

Common examples include $g(r)=r^2$ ($k$-means), $g(r)=r$ (continuous $k$-median / linear-loss prototype clustering), and the Huber loss $g_\tau(r)$, which interpolates between quadratic and linear penalties \cite{huber_robust_1964}. Note that we do \emph{not} require $g$ to be convex or strictly increasing globally, which allows our theory to extend to trimmed or truncated losses
\cite{garcia-escudero_robustness_1999,garcia-escudero_general_2008}.

\subsection{Loss increment geometry}
\label{subsec:increment}

To connect the loss function with the geometric separation of clusters, we analyze the cost of moving a point across a cluster boundary. Intuitively, if a point that belongs to a cluster at distance $r$ from its prototype is reassigned to a different prototype that is at least $\gamma$ farther away, the loss must increase by some amount. 

We quantify the smallest such increase through the \emph{uniform loss increment}.

\begin{definition}[Uniform Loss Increment]\label{def:increment}
For a radius $D \ge 0$ and a margin shift $\gamma > 0$, the uniform loss increment $\Delta_g(\gamma; D)$ is defined as:
\begin{equation}\label{eq:uniform-increment}
\Delta_g(\gamma; D) := \inf_{0 \le r \le D} \{ g(r + \gamma) - g(r) \}.
\end{equation}
\end{definition}

When $\Delta_g(\gamma;D)=0$ (e.g., for saturated or locally flat losses over the relevant range), no cost-based certificate can prevent misclassification, and our bounds correctly become vacuous through the factor $1/\Delta_g(\gamma;D)$. In particular, $\Delta_g(\gamma;D)>0$ holds for strictly increasing losses and for piecewise-linear/quadratic losses (e.g., Huber) on the relevant range.

Intuitively, $\Delta_g(\gamma; D)$ measures the smallest loss increase incurred when a point whose distance to its correct prototype is at most $D$ is reassigned to a prototype that is at least $\gamma$ farther away.
This quantity will later appear as the denominator of the clustering condition number that governs the stability guarantees in Section~\ref{sec:3}.

\medskip

\noindent\textbf{Standing Assumption.} Throughout, we additionally assume that $g$ is $L_g$-Lipschitz on the domain $[0, D_{\text{eff}} + \Delta_0]$, which covers the relevant range of distances for any near-optimal solution.

\subsection{Benchmark geometry}\label{subsec:2.2}

We adopt a nonparametric perspective and analyze clustering relative to a fixed geometric benchmark. Rather than assuming a generative model, we specify a reference partition and a set of associated prototype locations that serve as anchors for the cluster geometry. These objects define the separation and scale parameters that govern structural recoverability.

Let $\mathcal{C}^* = \{C_1^*, \dots, C_k^*\}$ be a reference partition of the data, and let $\theta^* = (\theta_1^*, \dots, \theta_k^*) \in \Theta^k$ be a set of associated prototypes.

\begin{remark}[Geometric reference vs.\ optimizer]
The benchmark pair $(\mathcal{C}^*,\theta^*)$ need not minimize the clustering objective. 
Instead, $\theta_j^*$ serves as a geometric anchor for the cluster $\mathcal{C}_j^*$. 
For example, $\theta^*$ may represent population means, medoids, or ground-truth centers in synthetic experiments.
\end{remark}

The stability of recovering $\mathcal{C}^*$ is governed by four geometric quantities, defined solely by the benchmark configuration $(\mathcal{C}^*,\theta^*)$ and the ambient metric space:
\begin{enumerate}
    \item \emph{Effective radius ($D_{\mathrm{eff}}$).} The maximum cluster radius relative to its anchor:
    \begin{equation}\label{eq:Deff}
        D_{\text{eff}} := \max_{j \in [k]} \max_{i \in C_j^*} d(x_i, \theta_j^*).
    \end{equation}

    \item \emph{Prototype separation ($\Delta_0$).} The minimum distance between distinct benchmark prototypes:
    \begin{equation}\label{eq:Delta0}
        \Delta_0 := \min_{j \neq l} d(\theta_j^*, \theta_l^*).
    \end{equation}

    \item \emph{Geometric margin ($\gamma$).} The slack between separation and twice the radius:
    \begin{equation}\label{eq:gamma-def}
        \gamma := \Delta_0 - 2 D_{\text{eff}}.
    \end{equation}

    \item \emph{Balance ($c_b$).} The minimum benchmark cluster proportion:
    \begin{equation}
        c_b := \min_{j \in [k]} |C_j^*|/n.
    \end{equation}
\end{enumerate}

Together, the quantities $(D_{\mathrm{eff}},\Delta_0,\gamma,c_b)$ summarize the geometric scale, separation, and balance of the benchmark configuration.

Throughout the paper, we work in the \emph{separable regime} where $\gamma > 0$. This condition ensures that the balls $B(\theta_j^*, D_{\text{eff}})$ are disjoint. More importantly, it provides a strict geometric margin that rules out ambiguous assignments at the benchmark scale. Specifically, since $d(\theta_j^*, \theta_l^*) \ge 2D_{\text{eff}} + \gamma$, the triangle inequality implies that for any $i\in C_j^*$ and $l\neq j$:
\begin{equation}\label{eq:benchmark-margin-ineq}
d(x_i,\theta_l^*) \ge d(\theta_j^*, \theta_l^*) - d(x_i, \theta_j^*) \ge (2D_{\text{eff}} + \gamma) - D_{\text{eff}} = D_{\text{eff}} + \gamma.
\end{equation}
Since also $d(x_i,\theta_j^*) \le D_{\text{eff}}$, it follows that
\begin{equation}\label{eq:benchmark-margin-diff}
d(x_i,\theta_l^*) \ge d(x_i,\theta_j^*) + \gamma.
\end{equation}
Thus, moving a point from its correct anchor to an incorrect one necessarily incurs an additional distance of at least $\gamma$ beyond its benchmark distance to the correct prototype. This margin will translate directly into a minimum loss increase through the increment quantity $\Delta_g$ defined in Section~\ref{subsec:increment}.
The geometry is illustrated in Figure \ref{fig:geometry}.

\begin{figure}[t]
\centering
\begin{tikzpicture}[scale=1.1]

\coordinate (theta1) at (0,0);
\coordinate (theta2) at (6,0);

\coordinate (xi) at (1.6,0);

\draw[dashed] (theta1) circle (2);

\fill (theta1) circle (2pt) node[below] {$\theta_j^*$};
\fill (theta2) circle (2pt) node[below] {$\theta_\ell^*$};
\fill (xi) circle (2pt) node[above] {$x_i$};

\draw[<->] (theta1) -- (xi) node[midway, above] {$r$};

\draw[<->] (xi) -- (theta2) node[midway, above] {$\ge r+\gamma$};

\node at (0,2.4) {$B(\theta_j^*,D_{\mathrm{eff}})$};

\end{tikzpicture}

\caption{
Geometric interpretation of the loss increment. 
A point $x_i$ in cluster $j$ lies at distance $r \le D_{\mathrm{eff}}$ from its anchor $\theta_j^*$. 
Any reassignment to another prototype $\theta_\ell^*$ increases the distance by at least $\gamma$, so that 
$d(x_i,\theta_\ell^*) \ge r+\gamma$. 
This distance gap induces a minimum loss increase of $\Delta_g(\gamma;D_{\mathrm{eff}})$.
}
\label{fig:geometry}
\end{figure}
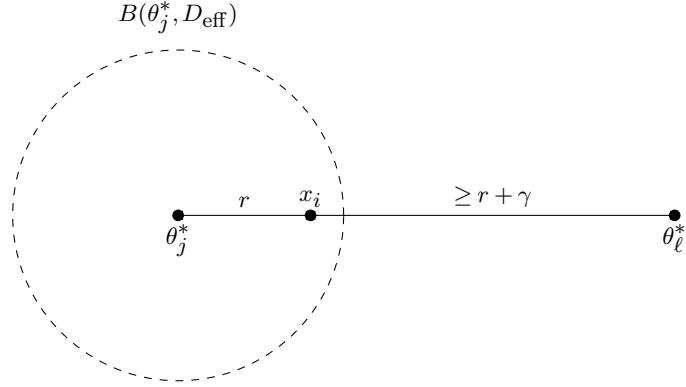


\begin{remark}[The separability condition in context]
\label{rem:separation-comparison}
The condition $\gamma > 0$ requires that the balls $B(\theta_j^*, D_{\mathrm{eff}})$ be pairwise disjoint, i.e., geometric separability. Several points of comparison are worth noting.
\begin{enumerate}
\item[(i)] When $\gamma = 0$, the increment $\Delta_g(0; D_{\mathrm{eff}}) = 0$ for any loss in our class, so the condition number $\kappa$ diverges and the stability bound becomes vacuous. This suggests that some form of strict separation is a genuine structural requirement for bounds of the type we study, though we do not claim $\gamma > 0$ is the weakest condition under which such bounds are possible.

\item[(ii)] In the Gaussian mixture literature, recovery guarantees typically require a signal-to-noise condition of the form $\Delta_0 \ge C\sigma\sqrt{k\log n}$ or $\Delta_0 \ge C\sigma k^{1/4}$, where $\sigma$ is the sub-Gaussian parameter \cite{lu_statistical_2016,ndaoud_sharp_2022}. Under sub-Gaussian tails the effective radius satisfies $D_{\mathrm{eff}} \lesssim \sigma\sqrt{d+\log n}$, so $\gamma > 0$ becomes $\Delta_0 \gtrsim \sigma\sqrt{d+\log n}$, comparable in spirit to those signal-to-noise conditions but stated deterministically and without a generative model.

\item[(iii)] The $(1+\alpha)$-perturbation resilience framework of \cite{balcan_clustering_nodate} requires the optimal clustering to stay invariant under multiplicative perturbations of the distance metric. The condition $\gamma > 0$ is logically independent of this. Perturbation resilience can hold without strict separation, and strict separation can hold without perturbation resilience. Both conditions identify instances where clustering is well-posed beyond optimization alone, but through different mechanisms.


\item[(iv)] The condition $\gamma > 0$ need not be large for the theory to be useful. For $k$-means the bound reads $p \lesssim (D_{\mathrm{eff}}/\gamma)^2 \cdot \delta$, so as the margin narrows the required optimization gap must tighten at rate $(\gamma/D_{\mathrm{eff}})^2$. For linear loss the dependence is only linear in $\gamma/D_{\mathrm{eff}}$. The theory remains informative at small $\gamma$ by trading geometric width for optimization precision; see Corollary~\ref{cor:kappa_bound} for the general form.
\end{enumerate}
\end{remark}

\begin{remark}[Scaling with Sample Size]
While our definitions are deterministic given the sample $S$, it is instructive to consider their behavior if points are drawn from a population distribution. If the distribution has compact support, $D_{\text{eff}}$ approaches a constant as $n \to \infty$. If the distribution has sub-Gaussian tails, $D_{\text{eff}}$ typically scales as $\sqrt{\log n}$ (up to scale and dimension factors). The condition number $\kappa$ defined in the next section will naturally capture these scaling laws, penalizing heavy-tailed geometry where the radius grows faster than the separation.
\end{remark}

\subsection{Optimality gap and misclassification}\label{subsec:2.3}

We now relate the algorithmic output of a clustering procedure to the benchmark geometry introduced above. Specifically, we distinguish between optimization accuracy, measured by the objective gap, and structural accuracy, measured by the misclassification rate relative to the benchmark partition.

Let $(\hat{\mathcal C},\hat\theta)$ be a candidate solution produced by an algorithm. 
Define the optimal objective value
\[
\mathsf{OPT}_n := \min_{\mathcal C,\theta} \mathcal L_n(\mathcal C,\theta).
\]
We measure optimization accuracy through the multiplicative gap $\delta \ge 0$ defined by
\begin{equation}
    \label{eq:mult-gap}
\mathcal L_n(\hat{\mathcal C},\hat\theta) \le (1+\delta)\mathsf{OPT}_n.
\end{equation}
This multiplicative optimality-gap formulation is common in the analysis of 
approximation algorithms for $k$-means clustering and spectral clustering
\cite{arthur_k-means_2006,balcan_approximate_2009,kanungo_local_2004,peng_partitioning_2015}.


To quantify recovery of the benchmark partition, we define the misclassification rate.

\begin{definition}[Misclassification rate]
\label{def:misclassification}
Let $\Pi_k$ denote the set of permutations of $[k]$. 
The misclassification rate of $\hat{\mathcal C}$ relative to $\mathcal C^*$ is
\[
p(\hat{\mathcal C},\mathcal C^*) 
:= 1-\max_{\pi\in\Pi_k}
\frac{1}{n}\sum_{j=1}^k |\hat C_j\cap C_{\pi(j)}^*|.
\]
\end{definition}

This metric represents the fraction of points that are assigned to an incorrect cluster, after optimally matching the cluster labels. Equivalently,
\begin{equation}\label{eq:misclass-rate-equiv}
p(\hat{\mathcal{C}}, \mathcal{C}^*) = \min_{\pi \in \Pi_k}\frac{1}{n}\sum_{j=1}^k \bigl|C_j^* \setminus \hat{C}_{\pi(j)}\bigr|.
\end{equation}

Our goal is to bound the structural error in terms of the optimization gap and the geometric parameters introduced earlier. In the next section we establish a stability inequality of the form
\[
p(\hat{\mathcal C},\mathcal C^*)
\le
\kappa(g,\text{geometry})\cdot\delta
+\text{approximation and displacement terms}.
\]

\section{Main stability results}\label{sec:3}

We now present the main theoretical guarantee linking the optimization gap
to the misclassification rate.
Our result relies on a single geometric quantity, the \emph{clustering condition number}, which captures the interplay between the within-cluster scale and the cross-cluster separation under the loss $g$.

\subsection{Condition number}\label{subsec:3.1}

Before deriving the bounds, we combine the geometric parameters introduced
in Section~\ref{sec:2} ($D_{\text{eff}}, \gamma$) into a dimensionless
quantity that governs stability. Intuitively, the difficulty of clustering
is determined by the ratio between the ``within-cluster scale'' and the
``cost of a mistake.'' The former is captured by $g(D_{\text{eff}})$,
while the latter is quantified by the uniform loss increment $\Delta_g$
(Definition~\ref{def:increment}). This leads to the following definition.

\begin{definition}[Clustering Condition Number]\label{def:condition_number}
For an admissible loss $g$ and a benchmark geometry characterized by margin $\gamma$ and radius $D_{\text{eff}}$, the condition number $\kappa$ is defined as
\begin{equation}\label{eq:kappa-def}
    \kappa(g,\gamma,D_{\text{eff}}) := \frac{g(D_{\text{eff}})}{\Delta_g(\gamma; D_{\text{eff}})}.
\end{equation}
\end{definition}

The term $\Delta_g(\gamma; D_{\text{eff}})$ represents the minimum loss
increase induced by moving a point across a geometric margin $\gamma$,
while $g(D_{\text{eff}})$ represents the benchmark within-cluster loss
scale. 
Their ratio therefore compares the typical within-cluster loss scale to the
minimal penalty incurred when a point crosses a cluster boundary.
A small $\kappa$ indicates a
well-conditioned problem in which separation dominates within-cluster
variability. Although the main stability inequality below is stated in
terms of the objective value $OPT_n$, this ratio naturally emerges when
the bound is expressed in a simplified condition-number form
(Corollary~\ref{cor:kappa_bound}).

\begin{remark}[Scale invariance]\label{rem:scale-invariance}
For homogeneous losses where $g(c r)=c^p g(r)$ (e.g., squared loss and absolute loss), $\kappa$ depends only on the relative geometry $\gamma/D_{\text{eff}}$. In particular, for $k$-means ($g(r)=r^2$), one has $\Delta_g(\gamma;D_{\text{eff}})=\gamma^2$ and hence $\kappa=(D_{\text{eff}}/\gamma)^2$. For linear loss ($g(r)=r$), $\Delta_g(\gamma;D_{\text{eff}})=\gamma$ and hence $\kappa=D_{\text{eff}}/\gamma$.
See also Section~\ref{subsec:4.1}. 
\end{remark}

\subsection{General stability inequality}\label{subsec:3.2}

Our main theorem states that any solution with an objective value close to the global optimum must be structurally close to the benchmark. To establish this, we first quantify the minimum penalty incurred by a single misclassified point (Lemma~\ref{lem:pointwise_cost}) and then aggregate these penalties to bound the total excess loss (Lemma~\ref{lem:aggregate_gap}). A key subtlety is that the bound involves the prototype displacement $\eta$, which must be controlled to keep the effective margin $\gamma-\eta$ positive. For standard objectives, this control is guaranteed in a \emph{small-gap regime} via the parameter-stability results in Section~\ref{subsec:5.3}; here we state the stability theorem conditionally on $\eta<\gamma$ and explicitly record this linkage.

Let $\delta \ge 0$ be the multiplicative optimality gap from \eqref{eq:mult-gap}, i.e.,
$\mathcal{L}_n(\hat{\mathcal{C}}, \hat{\theta}) \le (1+\delta)\mathsf{OPT}_n$.
Let $\delta_{\text{approx}} \ge 0$ denote the benchmark approximation error:
\begin{equation}\label{eq:delta-approx}
    \mathcal{L}_n(\mathcal{C}^*, \theta^*) = (1 + \delta_{\text{approx}})\,\mathsf{OPT}_n.
\end{equation}
Let $\eta$ denote the \emph{prototype displacement} under the optimal label matching $\pi$:
\begin{equation}\label{eq:eta-def}
    \eta := \min_{\pi\in\Pi_k}\max_{j\in[k]} d\bigl(\hat{\theta}_j,\theta^*_{\pi(j)}\bigr).
\end{equation}
Recall from Section~\ref{sec:2} that we work under the standing assumption that $g$ is $L_g$-Lipschitz on the relevant domain (covering both benchmark and candidate configurations).

\begin{remark}[Small-gap regime and displacement control]\label{rem:small_gap_eta}
The inequalities below require $\eta<\gamma$ so that the effective margin $\gamma-\eta$ is positive and $\Delta_g(\gamma-\eta;D_{\text{eff}})$ is informative.
For center-based objectives such as $k$-means (and for discrete prototype sets such as $k$-medoids), Section~\ref{subsec:5.3} shows that $\eta$ is intrinsically controlled by the optimization gap: in particular, $\eta \lesssim D_{\text{eff}}\sqrt{\delta+\delta_{\text{approx}}}$ for $k$-means and $\eta=0$ once $\delta$ is below a discrete swap threshold for $k$-medoids. Consequently, whenever $\delta$ is sufficiently small (relative to $\gamma/D_{\text{eff}}$ and the discrete gap, respectively), the condition $\eta<\gamma$ holds automatically.
\end{remark}

\begin{lemma}[Pointwise Cost of Misclassification]\label{lem:pointwise_cost}
Consider a point $x_i$ belonging to the benchmark cluster $C_j^*$. If a candidate partition $\hat{\mathcal{C}}$ assigns $x_i$ to a cluster $\hat{C}_l$ with $l \ne \pi(j)$ (i.e., $x_i$ is misclassified), and the candidate prototypes satisfy the displacement bound $\eta < \gamma$, then the loss contribution satisfies:
\begin{equation}\label{eq:pointwise-cost}
    g(d(x_i, \hat{\theta}_l)) - g(d(x_i, \theta_j^*)) \ge \Delta_g(\gamma - \eta; D_{\text{eff}}).
\end{equation}
\end{lemma}
\begin{proof}
By the definition of $\eta$ in \eqref{eq:eta-def}, we have $d(\hat{\theta}_l,\theta^*_{\pi^{-1}(l)})\le \eta$, hence
\[
d(x_i,\hat{\theta}_l)\ge d(x_i,\theta^*_{\pi^{-1}(l)})-\eta.
\]
Since $l\ne \pi(j)$, we have $\pi^{-1}(l)\ne j$. By the geometric margin property \eqref{eq:benchmark-margin-diff},
\[
d(x_i,\theta^*_{\pi^{-1}(l)}) \ge d(x_i,\theta_j^*)+\gamma,
\]
so $d(x_i,\hat{\theta}_l)\ge d(x_i,\theta_j^*)+\gamma-\eta$. Since $d(x_i,\theta_j^*)\le D_{\text{eff}}$ and $g$ is non-decreasing, it follows that
\[
g(d(x_i,\hat{\theta}_l)) - g(d(x_i,\theta_j^*)) \ge \inf_{0\le r\le D_{\text{eff}}}\{g(r+\gamma-\eta)-g(r)\}
= \Delta_g(\gamma-\eta;D_{\text{eff}}).
\]
\end{proof}

\begin{lemma}[Aggregate Excess Loss]\label{lem:aggregate_gap}
Let $p = p(\hat{\mathcal{C}}, \mathcal{C}^*)$ be the misclassification rate. The objective difference satisfies the lower bound:
\begin{equation}\label{eq:aggregate-gap}
    \mathcal{L}_n(\hat{\mathcal{C}}, \hat{\theta}) - \mathcal{L}_n(\mathcal{C}^*, \theta^*) \ge n \cdot p \cdot \Delta_g(\gamma - \eta; D_{\text{eff}}) - n \cdot L_g \eta.
\end{equation}
\end{lemma}
\begin{proof}
Let $M\subseteq [n]$ be the set of misclassified indices under the optimal matching $\pi$. Then $|M|/n = p$.
For each $i\in M$, Lemma~\ref{lem:pointwise_cost} yields a loss increase of at least $\Delta_g(\gamma-\eta;D_{\text{eff}})$.
For each $i\notin M$, the candidate assignment matches the benchmark assignment. By the Lipschitz property of $g$ and the bound $d(\hat{\theta}_{\pi(j)},\theta_j^*)\le \eta$, the loss decreases by at most $L_g \eta$.
Summing over all points:
\[
\mathcal{L}_n(\hat{\mathcal{C}}, \hat{\theta}) - \mathcal{L}_n(\mathcal{C}^*, \theta^*) \ge |M| \Delta_g(\gamma-\eta; D_{\text{eff}}) - (n-|M|) L_g \eta \ge np \Delta_g - n L_g \eta.
\]
\end{proof}

\begin{theorem}[Global stability]\label{thm:main_stability}
Assume the separable regime ($\gamma>0$) and that the uniform increment $\Delta_g(\gamma;D_{\text{eff}})>0$.
Let $(\hat{\mathcal{C}}, \hat{\theta})$ be any candidate solution satisfying \eqref{eq:mult-gap}.
Assume further that the displacement satisfies $\eta<\gamma$ (which holds automatically in the small-gap regime for standard objectives; see Remark~\ref{rem:small_gap_eta} and Section~\ref{subsec:5.3}).
Then the misclassification rate satisfies
\begin{equation}\label{eq:main_bound}
    p(\hat{\mathcal{C}}, \mathcal{C}^*) \;\le\; \underbrace{\frac{\mathsf{OPT}_n}{n\,\Delta_g(\gamma-\eta; D_{\text{eff}})}\,\bigl(\delta+\delta_{\text{approx}}\bigr)}_{\text{Optimization Error Term}}
    \;+\; \underbrace{\frac{L_g\,\eta}{\Delta_g(\gamma-\eta; D_{\text{eff}})}}_{\text{Displacement Term}}.
\end{equation}
\end{theorem}
\begin{proof}
We upper bound the objective difference using the optimality gaps. Write
\[
\mathcal{L}_n(\hat{\mathcal{C}},\hat{\theta}) - \mathcal{L}_n(\mathcal{C}^*,\theta^*)
= \bigl(\mathcal{L}_n(\hat{\mathcal{C}},\hat{\theta}) - \mathsf{OPT}_n\bigr)
- \bigl(\mathcal{L}_n(\mathcal{C}^*,\theta^*) - \mathsf{OPT}_n\bigr).
\]
Since $\mathcal{L}_n(\hat{\mathcal{C}},\hat{\theta}) - \mathsf{OPT}_n \le \delta\,\mathsf{OPT}_n$ and
$\mathcal{L}_n(\mathcal{C}^*,\theta^*) - \mathsf{OPT}_n = \delta_{\text{approx}}\,\mathsf{OPT}_n \ge 0$,
we have the conservative bound
\[
\mathcal{L}_n(\hat{\mathcal{C}},\hat{\theta}) - \mathcal{L}_n(\mathcal{C}^*,\theta^*)
\le \bigl(\delta+\delta_{\text{approx}}\bigr)\mathsf{OPT}_n.
\]
Combining this upper bound with the lower bound from Lemma~\ref{lem:aggregate_gap} and rearranging for $p$ yields \eqref{eq:main_bound}.
\end{proof}

The bound \eqref{eq:main_bound} is non-asymptotic and deterministic. Typically, the average optimal loss $\mathsf{OPT}_n/n$ is of the same order as the worst-case within-cluster loss $g(D_{\text{eff}})$. This yields the following simplified condition-number form.

\subsection{Condition-number bound}

Since $d(x_i,\theta^*_j) \le D_{\mathrm{eff}}$ for all $i\in C_j^*$,
we have
\[
L_n(\mathcal C^*,\theta^*) \le n\, g(D_{\mathrm{eff}}),
\]
and therefore $OPT_n/n \le L_n(\mathcal C^*,\theta^*)/n\le g(D_{\mathrm{eff}})$.
This yields the following condition-number form.

\begin{corollary}[Condition number form]\label{cor:kappa_bound}
If $\Delta_g(\gamma; D_{\mathrm{eff}})/
\Delta_g(\gamma-\eta; D_{\mathrm{eff}})\le C$ for some constant $C$, then
\begin{equation}\label{eq:kappa-cor}
    p(\hat{\mathcal{C}}, \mathcal{C}^*) \;\lesssim\; \kappa(g,\gamma,D_{\text{eff}}) \cdot \bigl(\delta + \delta_{\text{approx}}\bigr) \;+\; \frac{L_g\,\eta}{\Delta_g(\gamma; D_{\text{eff}})}.
\end{equation}
\end{corollary}

This corollary emphasizes the role of the condition number $\kappa$ when the displacement $\eta$ is small relative to the margin (specifically, when $\Delta_g(\gamma-\eta; D) \approx \Delta_g(\gamma; D)$).
In particular, for medoid-based formulations, if the benchmark prototypes are themselves feasible medoids and the algorithm returns the same medoids (so that $\eta=0$), the error is governed purely by $\kappa\cdot(\delta+\delta_{\text{approx}})$.

\subsection{Interpretation of the Stability Bound}

Theorem~\ref{thm:main_stability} decouples the statistical difficulty of the
clustering problem from the algorithmic task of minimizing the objective.
The bound separates three sources of error.

\begin{itemize}
\item \emph{Optimization accuracy ($\delta$).}
The algorithm only needs to deliver a near-optimal objective value.
The guarantee is algorithm-agnostic and does not depend on how the
solution $(\hat C,\hat\theta)$ is obtained.

\item \emph{Geometry ($\kappa$).}
The intrinsic difficulty of the clustering instance is captured by the
condition number $\kappa$. When the margin $\gamma$ is small relative to
the effective radius $D_{\mathrm{eff}}$, the increment
$\Delta_g(\gamma;D_{\mathrm{eff}})$ becomes small and $\kappa$ grows
large, making structural recovery unstable even for very small
optimization gaps.

\item \emph{Prototype displacement ($\eta$).}
The term linear in $\eta$ reflects the effect of shifting the prototype
locations, which effectively moves the cluster decision boundaries.
For center-based methods such as $k$-means, $\eta$ is not a free
parameter but is itself controlled by the optimization gap. In
Section~\ref{subsec:5.3} we show that for $k$-means,
\[
\eta \lesssim D_{\mathrm{eff}}\sqrt{\delta+\delta_{\mathrm{approx}}},
\]
so that in the small-gap regime the overall error behaves essentially
as $O(\kappa\delta)$.
\end{itemize}

\section{Examples}\label{sec:4}

In this section, we instantiate the general stability framework for specific choices of the loss generator $g$. We compute the resulting condition numbers for standard objectives ($k$-means and continuous $k$-median) and a robust variant (Huber), and we provide a sharp two-cluster analysis to highlight the trade-offs between separation and cluster balance.

Throughout this section, it is useful to distinguish the \emph{loss generator} from the \emph{prototype feasibility set}. The squared loss $g(r)=r^2$ with unconstrained prototypes corresponds to $k$-means, while the linear loss $g(r)=r$ with unconstrained prototypes corresponds to continuous $k$-median. The discrete prototype constraint of $k$-medoids is conceptually distinct and is not the object of the exact-recovery threshold in Section~\ref{subsec:4.3}.

\subsection{Condition numbers for common objectives}\label{subsec:4.1}

The two most common prototype clustering formulations correspond to the squared loss ($k$-means) and the linear loss (continuous $k$-median).

    

\emph{The squared loss} ($k$-means).
Let $g(r)=r^2$. The uniform increment is
\[
\Delta_g(\gamma;D)=\inf_{0\le r\le D}\bigl\{(r+\gamma)^2-r^2\bigr\}
=\inf_{0\le r\le D}\bigl\{\gamma^2+2\gamma r\bigr\}
=\gamma^2,
\]
attained at $r=0$ due to the strict convexity of the squared loss. The within-cluster scale is $g(D_{\text{eff}})=D_{\text{eff}}^2$. Hence
\begin{equation}\label{eq:kappa-means}
\kappa_{\text{means}}(\gamma,D_{\text{eff}})=\frac{D_{\text{eff}}^2}{\gamma^2}.
\end{equation}
The stability relation $p \lesssim (D_{\text{eff}}/\gamma)^2\,\delta$ shows that $k$-means becomes ill-conditioned when $\gamma$ is small relative to $D_{\text{eff}}$. The quadratic growth of $g$ also makes $D_{\text{eff}}$ sensitive to outliers.

\emph{The linear loss} (continuous $k$-median).
Let $g(r)=r$. Then $\Delta_g(\gamma;D)=\gamma$ and $g(D_{\text{eff}})=D_{\text{eff}}$, so
\begin{equation}\label{eq:kappa-medoids}
\kappa_{\text{med}}(\gamma,D_{\text{eff}})=\frac{D_{\text{eff}}}{\gamma}.
\end{equation}
Compared to $k$-means, the dependence on the relative margin is linear rather than quadratic. However, as the two-cluster analysis in Section~\ref{subsec:4.3} shows, linear objectives can be substantially more sensitive to severe imbalance when one demands \emph{exact} recovery of the benchmark.

The same condition-number scaling also applies to medoid-type formulations at the level of the loss increment itself; however, exact-recovery thresholds for $k$-medoids require a separate discrete argument and are not identified here with the continuous $k$-median result.

\begin{figure}[t]
\centering
\begin{tikzpicture}
\begin{axis}[
    width=11cm,
    height=7cm,
    xlabel={Margin ratio $\gamma/D$},
    ylabel={Condition number $\kappa$},
    ymode=log,
    xmin=0.1, xmax=2.5,
    samples=400,
    legend pos=north east
]

\addplot[thick,black] {1/x^2};
\addlegendentry{$k$-means}

\addplot[thick,black,dashed] {1/x};
\addlegendentry{$k$-medoids}

\addplot[
    thick,
    blue,
    mark=o,
    mark size=2pt,
    mark repeat=6
]
{ (0.5 - 0.125) / ( x <= 0.5 ? 0.5*x^2 : 0.5*x - 0.125 ) };
\addlegendentry{Huber, $\tau/D_{\mathrm{eff}}=0.5$}

\addplot[
    thick,
    red,
    mark=triangle,
    mark size=2pt,
    mark repeat=6
]
{ (1 - 0.5) / ( x <= 1 ? 0.5*x^2 : x - 0.5 ) };
\addlegendentry{Huber, $\tau/D_{\mathrm{eff}}=1$}

\addplot[
    thick,
    green!60!black,
    mark=square,
    mark size=2pt,
    mark repeat=6
]
{1/x^2};
\addlegendentry{Huber, $\tau/D_{\mathrm{eff}}=2$}

\end{axis}
\end{tikzpicture}
    \caption{Condition number $\kappa$ as a function of the relative margin $\gamma/D_{\mathrm{eff}}$. Huber loss (shown with markers, $\tau/D_{\mathrm{eff}}\in\{0.5,1,2\}$) interpolates between the quadratic scaling of $k$-means and the linear scaling of $k$-medoids, providing a tunable trade-off between stability and robustness. The $k$-means curve coincides with the Huber curve when $\tau/D_{\mathrm{eff}}=2$.}
    \label{fig:huber_kappa}

\end{figure}
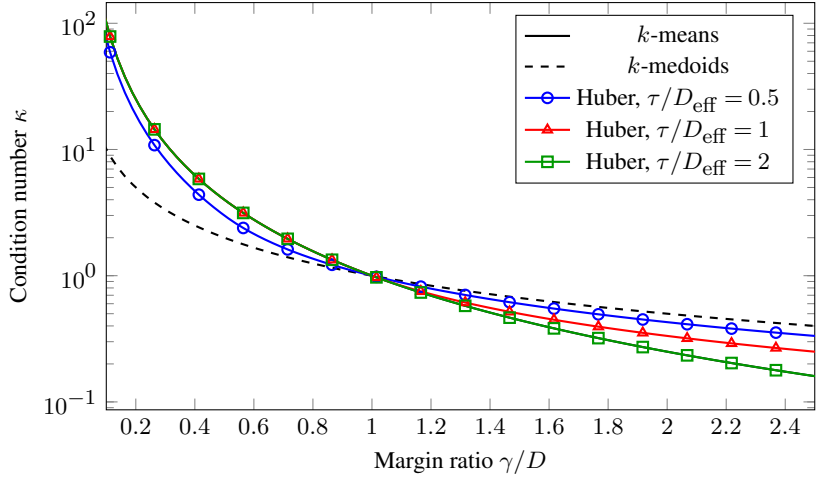

\subsection{Robustness via Huber loss}\label{subsec:4.2}

To interpolate between the curvature of squared loss and the robustness of linear loss, consider the Huber loss with tuning parameter $\tau>0$:
\begin{equation}\label{eq:huber-def}
g_\tau(r):=\begin{cases}
\frac{1}{2}r^2, & 0\le r\le \tau,\\
\tau r-\frac{1}{2}\tau^2, & r>\tau.
\end{cases}
\end{equation}
Since $g_\tau$ is convex, the increment $r\mapsto g_\tau(r+\gamma)-g_\tau(r)$ is non-decreasing in $r$. Therefore, the uniform loss increment satisfies, for any $D\ge 0$:
\begin{equation}\label{eq:huber-increment}
\Delta_{g_\tau}(\gamma;D)=g_\tau(\gamma)=
\begin{cases}
\frac{1}{2}\gamma^2, & \gamma\le \tau,\\
\tau\gamma-\frac{1}{2}\tau^2, & \gamma>\tau.
\end{cases}
\end{equation}
The resulting condition number is
\begin{equation}\label{eq:huber-kappa}
\kappa_\tau(\gamma,D_{\text{eff}})=\frac{g_\tau(D_{\text{eff}})}{g_\tau(\gamma)}.
\end{equation}

The following cases make the interpolation explicit (see Figure \ref{fig:huber_kappa}).
\begin{itemize}
\item If $D_{\text{eff}}\le \tau$ and $\gamma\le \tau$, then $\kappa_\tau=(D_{\text{eff}}/\gamma)^2$ (the $k$-means regime).
\item If $D_{\text{eff}}>\tau$ and $\gamma>\tau$, then
\[
\kappa_\tau(\gamma,D_{\text{eff}})
=\frac{\tau D_{\text{eff}}-\frac{1}{2}\tau^2}{\tau\gamma-\frac{1}{2}\tau^2}
\approx \frac{D_{\text{eff}}}{\gamma},
\]
recovering the linear-loss (continuous $k$-median) scaling up to constants.
\item In the mixed regimes (e.g., $D_{\text{eff}}>\tau$ but $\gamma\le\tau$), the numerator is linear while the denominator is quadratic, yielding intermediate conditioning.
\end{itemize}

A practical tuning principle suggested by \eqref{eq:huber-increment}--\eqref{eq:huber-kappa} is to choose the Huber threshold at the margin scale, $\tau \approx \gamma$. Then typical within-cluster residuals $r$ lie in the quadratic region $r \le \tau$, preserving the curvature of squared loss, while a boundary crossing increases the distance by roughly $\gamma$, moving the loss into the linear regime. This yields $k$-means-like behavior for moderate residuals while preventing boundary or outlier points from incurring excessively large quadratic penalties.

\subsection{Phase transition in a two-cluster model}\label{subsec:4.3}

We now specialize the geometry to a two-cluster setting to quantify the limits of exact recoverability and to isolate the role of imbalance.

Consider $k=2$ and suppose the benchmark clusters satisfy
$C_1^*\subset B(\theta_1^*,D)$ and $C_2^*\subset B(\theta_2^*,D)$ with
$d(\theta_1^*,\theta_2^*)=\Delta>2D$.
Let $c_b\in(0,1/2]$ be the balance coefficient, i.e., $c_b=\min\{|C_1^*|,|C_2^*|\}/n$.
In this model, $D_{\text{eff}}=D$ and the margin is $\gamma=\Delta-2D$.

We ask how large $\Delta/D$ must be to guarantee that the benchmark partition is the unique minimizer of the profiled clustering objective, up to label swapping.
That is, under what separation is $\,\mathcal C^*\,$ the unique minimizer of
\[
\mathcal C \mapsto \min_{\theta}\mathcal L_n(\mathcal C,\theta)
\]
up to label swapping?

The next theorem gives sufficient conditions for exact recovery.
For squared loss, the statement concerns standard $k$-means.
For linear loss, the statement below is formulated for the \emph{continuous one-dimensional (collinear) $k$-median problem}; the corresponding discrete $k$-medoids problem is not covered by this theorem.
These thresholds are order-wise tight (up to universal constants);
see Appendix~\ref{app:sec4} for the lower bounds.

\begin{theorem}[Sufficient thresholds for exact recovery]\label{thm:two_cluster_thresholds}
Under the two-ball model with balance $c_b$, the following conditions are sufficient for exact recovery of the benchmark partition as the unique minimizer of the profiled objective $\mathcal C\mapsto \min_\theta \mathcal L_n(\mathcal C,\theta)$, up to label swapping:
\begin{enumerate}
    \item[(i)] \textbf{For $k$-means:}
    \begin{equation}\label{eq:kmeans_threshold}
        \frac{\Delta}{D} > 2 + \frac{2}{\sqrt{c_b}}.
    \end{equation}
    \item[(ii)] \textbf{For one-dimensional continuous $k$-median (collinear two-ball model):}
    \begin{equation}\label{eq:kmedian_threshold}
        \frac{\Delta}{D} > 2 + \frac{1}{c_b}.
    \end{equation}
\end{enumerate}
\end{theorem}

The sufficiency follows from establishing a deterministic macro-failure lower bound and a micro-mixing swap argument (detailed in Appendix~\ref{app:proof_thm41_full}). In the linear-loss case, the proof is carried out in the collinear one-dimensional model, where the absolute-deviation objective admits an exact merge-penalty argument. The order-wise tightness of the scaling is established via adversarial configurations in Appendix~\ref{app:tightness}.

    

\begin{figure}[t]
\centering
\begin{tikzpicture}
\begin{axis}[
    width=11cm,
    height=7cm,
    xlabel={Normalized separation $\Delta/D$},
    ylabel={Balance coefficient $c_b$},
    xmin=2, xmax=10,
    ymin=0, ymax=0.5,
    domain=0.01:0.5,
    samples=400,
    legend pos=south west,
]

\addplot[name path=kmeans, thick, blue]
({2 + 2/sqrt(x)}, {x});
\addlegendentry{$k$-means boundary: $2+2/\sqrt{c_b}$}

\addplot[name path=kmedoids, thick, red, dashed]
({2 + 1/x}, {x});
\addlegendentry{1D continuous $k$-median boundary: $2+1/{c_b}$}

\addplot[name path=right, draw=none]
(10,0.01) -- (10,0.5);

\addplot[blue, fill opacity=0.2, draw=none] fill between[of=kmeans and right];

\addplot[red, fill opacity=0.2, draw=none] fill between[of=kmedoids and right];

\addplot[only marks, mark=*, mark size=1.5pt] coordinates {(6,0.25)};

\node[anchor=west, font=\small] at (axis cs:6,0.25) {$\,(6,\,0.25)$};

\end{axis}
\end{tikzpicture}

\caption{
Two-cluster phase diagram in the geometric two-ball model.
Curves show the sufficient exact-recovery boundaries from
Theorem~\ref{thm:two_cluster_thresholds}.
Regions to the right of each curve correspond to parameter values for which the benchmark partition is guaranteed to be the unique minimizer of the profiled objective $C \mapsto \min _\theta L_n(C, \theta)$, up to label swapping.
}
\label{fig:phase_diagram}
\end{figure}
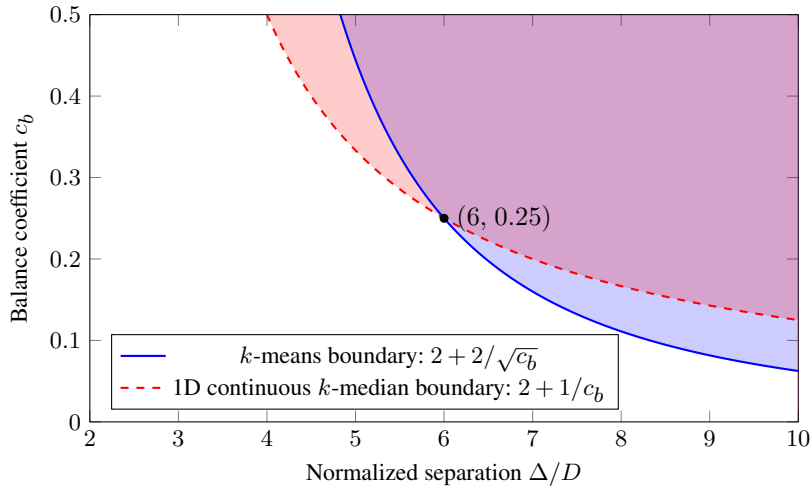

In the two-ball model, $\gamma/D=(\Delta/D)-2$. The thresholds in
Theorem~\ref{thm:two_cluster_thresholds} can therefore be expressed
directly in terms of the condition number. For $k$-means, \eqref{eq:kmeans_threshold} is equivalent to
$\gamma/D>2/\sqrt{c_b}$, that is,
\[
\kappa_{\text{means}}(\gamma,D)=(D/\gamma)^2<c_b/4.
\]
For one-dimensional continuous $k$-median, \eqref{eq:kmedian_threshold} is equivalent to
$\gamma/D>1/c_b$, that is,
\[
\kappa_{\text{med}}(\gamma,D)=D/\gamma<c_b.
\]
Thus exact recoverability in this model requires
$\kappa_{\text{means}}<c_b/4$ for squared loss and
$\kappa_{\text{med}}<c_b$ for one-dimensional continuous linear loss.

Comparing \eqref{eq:kmeans_threshold} and \eqref{eq:kmedian_threshold}
(as visualized in Figure~\ref{fig:phase_diagram}) reveals a qualitative
difference in sensitivity to imbalance. As $c_b\to0$, the separation
required for one-dimensional continuous $k$-median scales as $1/c_b$, whereas for $k$-means it
scales as $1/\sqrt{c_b}$.

As visualized in Figure~\ref{fig:phase_diagram}, the two sufficient
phase boundaries intersect at $c_b=0.25$. Thus their ordering is not
uniform over the full range $c_b\in(0,1/2]$. The substantive distinction
appears in the severe-imbalance regime. As $c_b\to0$, the separation
required under one-dimensional continuous linear loss scales as $1/c_b$, whereas under squared loss
it scales only as $1/\sqrt{c_b}$.

This phase transition analysis provides a theoretical guide for
selecting clustering objectives. Squared loss is advantageous under
strong imbalance when the within-cluster scale is well controlled,
whereas linear or other robust losses are preferable when heavy tails or outliers
inflate the effective radius. The theorem above should therefore be read as a sharp comparison between squared loss and continuous linear loss in the two-ball model, rather than as a complete characterization of the discrete $k$-medoids problem.

\section{Refined stability analysis: Local robustness and geometry}\label{sec:5}

Section~\ref{sec:3} establishes a global stability guarantee, showing that near-optimal objective values imply small total misclassification mass.
However, this bound treats all points symmetrically and therefore does not
reveal \emph{where} errors occur in the data geometry.

In practice, clustering errors are rarely distributed uniformly.
Points deep inside clusters are typically much easier to classify correctly
than points near cluster boundaries. Intuitively, interior points enjoy
larger geometric margins, while boundary points lie near decision surfaces
where small prototype shifts may change assignments.

In this section we formalize this intuition by refining the global stability
analysis. We show that the geometry naturally decomposes each cluster into a
\emph{core region} and a \emph{boundary belt}. Points in the core benefit from
an enhanced margin and therefore admit stronger stability guarantees.
In particular, sufficiently deep cores can be certified to have
\emph{zero misclassification error}, even when the overall solution is only
near-optimal.

We then address a second issue left implicit in Section~3, namely the prototype
displacement parameter $\eta$. For standard clustering objectives such as
$k$-means and $k$-medoids, this displacement is not an independent quantity
but is itself controlled by the optimization gap. Establishing this link
shows that the stability bound effectively depends on a single parameter in
the small-gap regime.

\subsection{Spatial concentration: The core--belt decomposition}\label{subsec:5.1}

We quantify the ``depth'' of a point within a benchmark cluster $C_j^*$ by its slack relative to the benchmark radius constraint. For any depth parameter $s \in [0, D_{\text{eff}})$, define the \emph{depth-$s$ core} and the \emph{boundary belt} as:
\begin{align}
    \text{Core}(s) &:= \bigcup_{j=1}^k \bigl\{ i \in C_j^* : d(x_i, \theta_j^*) \le D_{\text{eff}} - s \bigr\}, \\
    \text{Belt}(s) &:= [n] \setminus \text{Core}(s).
\end{align}

    

Moving deeper into a cluster increases the effective separation from other
clusters. As a point moves farther from competing prototypes while remaining
close to its own anchor, the geometric margin against incorrect assignments
grows. The next lemma quantifies this margin amplification.

\begin{lemma}[Anchor-depth enhanced margin]\label{lem:core_margin}
Under the separable regime ($\gamma > 0$), if a point $x_i$ belongs to the core of cluster $j$ at depth $s$, then for any other benchmark prototype $\theta_l^*$ ($l \ne j$),
\begin{equation}\label{eq:enhanced_margin}
    d(x_i, \theta_l^*) \ge d(x_i, \theta_j^*) + \gamma + 2s.
\end{equation}
\end{lemma}
\begin{proof}
By the triangle inequality and the separation condition $d(\theta_j^*, \theta_l^*) \ge 2D_{\text{eff}} + \gamma$, we have
\[
d(x_i, \theta_l^*) \ge d(\theta_j^*, \theta_l^*) - d(x_i, \theta_j^*)
\ge (2D_{\text{eff}} + \gamma) - (D_{\text{eff}} - s) = D_{\text{eff}} + \gamma + s.
\]
Subtracting $d(x_i, \theta_j^*) \le D_{\text{eff}} - s$ from the right-hand side yields
$d(x_i, \theta_l^*) - d(x_i, \theta_j^*) \ge \gamma + 2s$, proving \eqref{eq:enhanced_margin}.
\end{proof}

The geometric intuition is illustrated in Figure~\ref{fig:corebelt}.

\begin{figure}[t]
\centering
\begin{tikzpicture}[scale=1.2]

\coordinate (c1) at (-2,0);
\coordinate (c2) at (2,0);

\def\core{1}
\def\Deff{1.4}

\fill[gray!25] (c1) circle (\core);
\fill[gray!25] (c2) circle (\core);

\draw[dashed, thick] (c1) circle (\Deff);
\draw[dashed, thick] (c2) circle (\Deff);

\fill (c1) circle (2pt);
\fill (c2) circle (2pt);

\node[below] at (c1) {$\theta_1^*$};
\node[below] at (c2) {$\theta_2^*$};

\node at (-2,1.8) {\textbf{Cluster 1}};
\node at (2,1.8) {\textbf{Cluster 2}};

\fill (-2.4,0.3) circle (1.5pt);
\fill (-2.2,-0.4) circle (1.5pt);
\fill (-1.7,0.4) circle (1.5pt);
\fill (-1.6,-0.2) circle (1.5pt);

\coordinate (b1) at (-0.8,0.5);
\coordinate (b2) at (-0.7,-0.4);
\coordinate (b3) at (-0.9,0.1);

\fill (b1) circle (1.5pt);
\fill (b2) circle (1.5pt);
\fill (b3) circle (1.5pt);

\draw[thin] (c2) -- (b1);
\draw[thin] (c2) -- (b2);
\draw[thin] (c2) -- (b3);

\fill (1.7,0.4) circle (1.5pt);
\fill (2.3,-0.3) circle (1.5pt);
\fill (2.1,0.5) circle (1.5pt);
\fill (1.9,-0.4) circle (1.5pt);

\node at (-2,0.7) {\small core};
\node at (2,0.7) {\small core};

\node at (0,-1) {\small boundary belt};

\end{tikzpicture}

\caption{Core--Belt decomposition. Shaded disks represent cluster cores
($d(x,\theta_j^*) \le D_{\mathrm{eff}}-s$), while dashed rings indicate
the boundary belt where misclassifications may occur. Points on the
right side of Cluster~1 lie closer to the competing prototype
$\theta_2^*$ and are therefore more susceptible to reassignment.}
\label{fig:corebelt}
\end{figure}
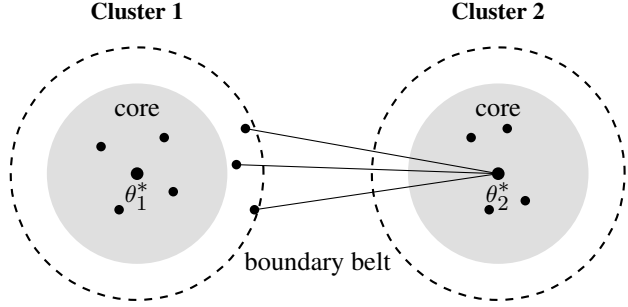

\subsection{Exact recovery of cluster cores}\label{subsec:5.2}

Lemma~\ref{lem:core_margin} implies that core points effectively face a larger margin $\gamma+2s$. Combining this with the prototype displacement $\eta$, the effective margin for a core point is $\gamma - \eta + 2s$. This motivates the \emph{local condition number}:
\begin{equation}\label{eq:kappa-local}
    \kappa(s) := \frac{g(D_{\text{eff}})}{\Delta_g(\gamma - \eta + 2s; D_{\text{eff}} - s)}.
\end{equation}
We now apply the logic of Theorem~\ref{thm:main_stability} to the restricted subset of core points. Since the global optimum $\mathsf{OPT}_n$ upper-bounds the optimum of any restricted sub-problem, keeping $\mathsf{OPT}_n$ in the numerator yields a conservative but valid localization bound.

\begin{proposition}[Local stability bound]\label{prop:core_stability}
For any solution satisfying the conditions of Theorem~\ref{thm:main_stability}, let $p_{\text{core}}(s)$ denote the fraction of misclassified points within the core (normalized by $n$). Then:
\begin{equation}\label{eq:core_bound}
    p_{\text{core}}(s) \;\le\; \frac{\mathsf{OPT}_n}{n \Delta_g(\gamma - \eta + 2s; D_{\text{eff}} - s)} (\delta + \delta_{\text{approx}}) + \frac{L_g \eta}{\Delta_g(\gamma - \eta + 2s; D_{\text{eff}} - s)}.
\end{equation}
\end{proposition}

\begin{proof}[Sketch] By Lemma~\ref{lem:core_margin}, points in the depth-$s$ core enjoy an enhanced geometric margin bounded below by $\gamma + 2s$, and their distance to the benchmark prototype is at most $D_{\mathrm{eff}} - s$. Applying the pointwise loss-increment lower bound solely to the misclassified points within the core, and conservatively dropping the non-negative penalty from any misclassified points in the belt, we obtain an aggregate excess loss lower bound. Comparing this to the global objective upper bound $(\delta + \delta_{\mathrm{approx}})\mathsf{OPT}_n$ directly yields the local stability certificate. The detailed derivation is provided in Appendix~\ref{app:refined_proofs}.
\end{proof}

A complete proof is given in Appendix~\ref{app:refined_proofs}.

This result provides a certificate for \emph{Zero-Error Cores}. Since the misclassification rate changes in increments of $1/n$, if the right-hand side of \eqref{eq:core_bound} is strictly less than $1/n$, then strictly fewer than one point can be misclassified in the core, hence $p_{\text{core}}(s) = 0$.

\subsection[Parameter stability and control of eta]{Parameter stability and control of $\eta$}\label{subsec:5.3}

The preceding analysis refines the spatial distribution of misclassification. We now ask how large the prototype displacement $\eta$ can be for a near-optimal solution.

The displacement parameter $\eta$, introduced in Section~\ref{sec:3}, measures how far the candidate prototypes move relative to the benchmark anchors. For standard clustering objectives, this displacement is itself controlled by the optimization gap, so $\eta$ does not represent an independent source of instability.

\begin{assumption}[Local quadratic growth for $k$-means]\label{ass:local-qg}
There exist constants $c_{\mathrm{qg}}>0$ and a neighborhood $\mathcal{U}$ of $\theta^*$ such that, for all $\theta\in\mathcal{U}$, the empirical objective satisfies:
\begin{equation}\label{eq:local-qg}
\mathcal{L}_n(\mathcal{C}^*,\theta)-\mathcal{L}_n(\mathcal{C}^*,\theta^*)
\;\ge\;
c_{\mathrm{qg}} \sum_{j=1}^k |C_j^*| \,\|\theta_j-\theta_j^*\|^2 .
\end{equation}
\end{assumption}

\begin{proposition}[Control of displacement]\label{prop:eta_control}
Assume the separable regime and that the benchmark approximation error $\delta_{\text{approx}}$ is defined as in \eqref{eq:delta-approx}.
\begin{enumerate}
    \item[(i)] {For $k$-means (strong local growth):}
    Suppose $g(r)=r^2$ and Assumption~\ref{ass:local-qg} holds, and that the matched prototypes $\hat{\theta}$ lie in the neighborhood $\mathcal{U}$. Then
    \begin{equation}\label{eq:eta-kmeans}
        \eta \;\lesssim\; D_{\text{eff}}\,\sqrt{\delta+\delta_{\text{approx}}}\,,
    \end{equation}
    where the implied constant depends on $c_{\mathrm{qg}}$ and the balance coefficient $c_b$.

    \item[(ii)] {For $k$-medoids (discreteness):}
    Suppose $g(r)=r$ and the prototype space $\Theta$ is finite (e.g., $\Theta=S=\{x_1,\dots,x_n\}$). Let $\Delta_{\min}$ denote the minimal positive cost increment induced by swapping a prototype. If
    \begin{equation}\label{eq:eta-medoids-gap}
        (\delta+\delta_{\text{approx}})\,\mathsf{OPT}_n < \Delta_{\min},
    \end{equation}
    then the matched prototypes must coincide with the benchmark prototypes, hence $\eta=0$.
\end{enumerate}
\end{proposition}

\begin{proof}[Sketch]
For (i), optimality implies $\mathcal{L}_n(\hat{\mathcal{C}},\hat{\theta})-\mathcal{L}_n(\mathcal{C}^*,\theta^*)\le(\delta+\delta_{\text{approx}})\mathsf{OPT}_n$. Under $k$-means, applying the local growth inequality \eqref{eq:local-qg} to the restricted benchmark partition yields a quadratic bound on $\|\hat{\theta}-\theta^*\|$, which rearranges to \eqref{eq:eta-kmeans}. For (ii), the discrete nature of the search space implies that any solution with cost difference smaller than the minimal gap $\Delta_{\min}$ cannot reside on a different prototype configuration.
\end{proof}

A complete proof is given in Appendix~\ref{app:refined_proofs}.

\subsection{Geometry of the near-optimal set: The ``tube''}\label{subsec:5.4}

Finally, we characterize the set of \emph{all} near-optimal solutions. For two partitions $\mathcal{C},\mathcal{C}'$, define the (label-invariant) Hamming distance by
\begin{equation}\label{eq:ham-def}
d_{\mathrm{Ham}}(\mathcal{C},\mathcal{C}')
:= 1-\max_{\pi\in\Pi_k}\frac{1}{n}\sum_{j=1}^k |C_j\cap C'_{\pi(j)}|.
\end{equation}

\begin{theorem}[Hamming tube]\label{thm:tube}
Let $(\hat{\mathcal{C}}^{(1)}, \hat{\theta}^{(1)})$ and $(\hat{\mathcal{C}}^{(2)}, \hat{\theta}^{(2)})$ be any two candidate solutions that are $(1+\delta)$-near-optimal, with aligned displacements $\eta_1,\eta_2$. Assume $\max\{\eta_1,\eta_2\}<\gamma$. Then:
\begin{equation}\label{eq:tube-bound}
    d_{\mathrm{Ham}}\!\bigl(\hat{\mathcal{C}}^{(1)}, \hat{\mathcal{C}}^{(2)}\bigr)
    \;\lesssim\;
    2 \kappa \cdot (\delta + \delta_{\text{approx}})
    \;+\;
    \frac{L_g(\eta_1+\eta_2)}{\Delta_g(\gamma-\max\{\eta_1,\eta_2\}; D_{\text{eff}})}.
\end{equation}
\end{theorem}
\begin{proof}[Sketch]
The Hamming distance satisfies the triangle inequality modulo permutations. We bound $d_{\mathrm{Ham}}(\hat{\mathcal{C}}^{(1)}, \hat{\mathcal{C}}^{(2)}) \le d_{\mathrm{Ham}}(\hat{\mathcal{C}}^{(1)}, \mathcal{C}^*) + d_{\mathrm{Ham}}(\hat{\mathcal{C}}^{(2)}, \mathcal{C}^*)$. Applying Theorem~\ref{thm:main_stability} to each term yields the result.
\end{proof}

A complete proof is given in Appendix~\ref{app:refined_proofs}.

\begin{remark}[Instability as a diagnostic]
Theorem~5.5 implies that all near-optimal solutions must lie within a small
Hamming neighborhood of the benchmark partition when the condition number
$\kappa$ is moderate and the optimization gap is small.
Consequently, if repeated runs of a randomized clustering algorithm
(e.g., different $k$-means initializations) produce partitions that have
similar objective values but large pairwise Hamming distances, this is
strong evidence that the clustering instance is \emph{ill-conditioned}.
In such regimes the loss landscape is nearly flat along directions that
change the partition, so multiple structurally distinct clusterings can
achieve comparable objective values.
\end{remark}

\begin{remark}[Partial separation and instance-adaptive margins]
\label{rem:partial-separation}
The core--belt decomposition also provides a natural answer to the concern that $\gamma > 0$ is a strong global assumption. Consider a dataset where most points lie well inside their cluster cores but a fraction $\beta$ lie near the boundaries, so $|\mathrm{Belt}(s)|/n = \beta$ for some depth $s > 0$. Even when the global margin $\gamma$ is small, the local condition number $\kappa(s)$ for the core may be moderate, and Proposition~\ref{prop:core_stability} then certifies exact recovery for those interior points. Total misclassification is bounded by the $\beta$-fraction of boundary points plus a $\kappa(s) \cdot \delta$ contribution from the better-conditioned interior. Structural guarantees therefore remain meaningful for the bulk of the data even when the global geometry is tight. One can read this as a soft-margin guarantee, where the effective margin is point-specific rather than uniform.
\end{remark}

\section{Operational implications and diagnostics}
\label{sec_diagnostics}
 In this section, we discuss the operational implications of the theory. 

\subsection{Data-driven stability certificates}\label{subsec:6.4}

While our bounds involve benchmark quantities $(\gamma,D_{\text{eff}})$, one can construct a conservative diagnostic certificate from observable proxies.
Here we propose a diagnostic procedure.

Given a candidate solution $(\hat{\mathcal{C}},\hat{\theta})$ with objective value $\hat{\mathcal{L}}$ and a conservative guard $\alpha\in(0,1)$ (e.g., $\alpha=0.2$), the procedure has four steps. In practice, $\alpha$ acts as a conservative guardrail against finite-sample variations in estimating the geometric margin.
Concretely, we run the diagnostic using the effective margin $(1-\alpha)\widehat{\gamma}$ (equivalently, we require a safety buffer of $\alpha\widehat{\gamma}$).
We recommend setting $\alpha\in[0.1,0.2]$ as a robust default, and reporting a simple sensitivity check over $\alpha\in\{0.1,0.2\}$ when the instance appears close to the boundary.

\begin{enumerate}
    \item \emph{Geometric proxies.} Compute the empirical within-cluster radius
    \[
    \hat{D} := \max_{j\in[k]}\max_{i\in\hat{C}_j} d(x_i,\hat{\theta}_j),
    \]
    and separation
    \[
    \hat{\Delta} := \min_{j\neq l} d(\hat{\theta}_j,\hat{\theta}_l).
    \]

    \item \emph{Guarded condition number.} Estimate the empirical margin $\hat{\gamma}:=(\hat{\Delta}-2\hat{D})_+$. Shrink the margin by $\alpha$ and compute
    \[
    \hat{\kappa} := \frac{g(\hat{D})}{\Delta_g((1-\alpha)\hat{\gamma};\hat{D})}.
    \]
    If $\hat{\gamma}=0$, set $\hat{\kappa}:=\infty$ (certificate is vacuous).

    \item \emph{Optimization gap proxy.} Estimate $\hat{\delta}$ empirically, e.g., via a best-of-$R$ multi-start heuristic,
    \[
    \hat{\delta}
    := \frac{\hat{\mathcal{L}} - \min_{r\le R}\hat{\mathcal{L}}^{(r)}}{\min_{r\le R}\hat{\mathcal{L}}^{(r)}},
    \]
    where $\hat{\mathcal{L}}^{(r)}$ denotes the objective value of the $r$-th randomized restart.

    \item \emph{Global certificate.} The product $\hat{p}_{\text{cert}} := \hat{\kappa}\cdot\hat{\delta}$ serves as a conservative certificate for the misclassification mass.
\end{enumerate}

\begin{remark}[Robust proxy radii]\label{rem:robust_proxy}
Since $\hat{D}$ is a maximum, it can be inflated by a small number of extreme points in heavy-tailed data. One may replace $\hat{D}$ by a high-quantile within-cluster radius (e.g., $95\%$ quantile) to obtain a \emph{trimmed} certificate that aligns with robust objectives such as Huber or medoids.
\end{remark}

\subsection{Stability of downstream functionals}\label{subsec:6.5}

The structural guarantees on $p$ imply stability for any downstream functional that is Lipschitz with respect to the label-invariant Hamming distance, which coincides with the misclassification rate $p(\hat{\mathcal{C}},\mathcal{C}^*)$ in Definition~\ref{def:misclassification}.
Let $T(\mathcal{C})$ be such a functional (e.g., cluster sizes, balance, or other partition-only summaries). Then
\[
|T(\hat{\mathcal{C}}) - T(\mathcal{C}^*)|
\;\le\; L_T \cdot p(\hat{\mathcal{C}}, \mathcal{C}^*),
\]
for some Lipschitz constant $L_T$. Hence, if the benchmark partition possesses desirable structural properties (e.g., balance), any stable near-optimal solution preserves them up to $O(\kappa\delta)$. Econometric implications of these structural properties are discussed elsewhere.

\section{Extensions}\label{sec:6}

The stability principles established above rely on a unified geometric abstraction. In this section, we show that this abstraction naturally extends to heterogeneous objectives, hierarchical structures, and dynamic settings.

\subsection{Heterogeneity via envelope condition numbers}\label{subsec:6.1}

In applications using instance-specific losses (e.g., adaptive Huber thresholds or weighted clustering \cite{inaba_applications_1994}), let each point $x_i$ be associated with an admissible loss $g_i$. The total objective is
\begin{equation}\label{eq:hetero_obj}
    \mathcal{L}_n(\mathcal{C}, \theta) \;=\; \sum_{j=1}^k \sum_{i \in C_j} g_i\!\bigl(d(x_i, \theta_j)\bigr).
\end{equation}
Stability is governed by the \emph{worst-case} capability of the loss mixture to penalize margin crossings versus its \emph{worst-case} internal scale.

\begin{definition}[Envelope quantities]\label{def:envelope}
For $\gamma>0$ and $D\ge 0$, define the \emph{uniform lower envelope increment} and the \emph{uniform upper envelope scale} by
\begin{align}
    \underline{\Delta}(\gamma; D) &:= \inf_{i \in [n]} \Delta_{g_i}(\gamma; D), \\
    \overline{G}(D) &:= \sup_{i \in [n]} g_i(D).
\end{align}
Let $\overline{L}$ denote a uniform Lipschitz bound over the relevant domain (consistent with the standing assumption in Section~\ref{sec:2}):
\begin{equation}
    \overline{L} \;:=\; \sup_{i \in [n]} \operatorname{Lip}\!\bigl(g_i \text{ on } [0,\, D_{\text{eff}}+\Delta_0]\bigr).
\end{equation}
\end{definition}

These quantities capture the worst-case behavior of the heterogeneous loss family,
specifically the smallest possible margin penalty across losses (through
$\underline{\Delta}(\gamma;D)$) and the largest loss value at radius $D$
(through $\overline{G}(D)$).

\begin{remark}[Weighted losses as a special case]\label{rem:weighted}
Weighted clustering can be embedded by setting $g_i(r)=w_i\,g(r)$ with weights $w_i>0$. Then
\[
\underline{\Delta}(\gamma;D)=(\inf_i w_i)\,\Delta_g(\gamma;D),
\qquad
\overline{G}(D)=(\sup_i w_i)\,g(D),
\]
so the effective condition number inherits the weight ratio $(\sup_i w_i)/(\inf_i w_i)$.
\end{remark}

\begin{corollary}[Heterogeneous stability]\label{cor:heterogeneous}
Under the heterogeneous setting, let $\eta$ be the prototype displacement and assume $\underline{\Delta}(\gamma-\eta;D_{\text{eff}})>0$. Any $(1+\delta)$-near-optimal solution satisfies
\begin{equation}\label{eq:hetero_bound}
    p \;\le\; \frac{\mathsf{OPT}_n}{n\,\underline{\Delta}(\gamma-\eta; D_{\text{eff}})}\,(\delta+\delta_{\text{approx}})
    \;+\; \frac{\overline{L}\,\eta}{\underline{\Delta}(\gamma-\eta; D_{\text{eff}})}.
\end{equation}
If additionally $\mathsf{OPT}_n/n \lesssim \overline{G}(D_{\text{eff}})$, then \eqref{eq:hetero_bound} takes the condition-number form
$p \lesssim \kappa_{\text{het}}(\delta+\delta_{\text{approx}})+ \overline{L}\eta/\underline{\Delta}(\gamma-\eta;D_{\text{eff}})$,
where $\kappa_{\text{het}} := \overline{G}(D_{\text{eff}})/\underline{\Delta}(\gamma-\eta;D_{\text{eff}})$.
\end{corollary}

\begin{proof}[Sketch]
    The result mirrors the global stability bound (Theorem~\ref{thm:main_stability}). By uniformly lower-bounding the pointwise cross-margin penalty using the lower envelope $\underline{\Delta}$ and upper-bounding the loss deviation using the uniform Lipschitz constant $\overline{L}$, the aggregate excess loss inequality holds identically. Detailed tracking of the parameters is provided in Appendix~\ref{app:extensions_proofs}.
\end{proof} 

\subsection{Hierarchical and multi-resolution benchmarks}\label{subsec:6.2}

Hierarchical clustering can be viewed as a sequence of flat clustering problems at varying scales $l=1,\dots,L$, each with geometric parameters $(\gamma_l, D_l)$ and a level-wise optimization gap $\delta_l$.

Our main theorem applies \emph{level-wise}. Assume the hierarchy is \emph{geometrically consistent}, meaning there exists $\rho>0$ such that $\gamma_l \ge \rho D_l$ for all $l$. Then the level-wise condition numbers $\kappa_l$ are uniformly bounded by some $\kappa_{\max}$.

\begin{remark}[Error metric for trees]\label{rem:tree_metric}
Let $p_l$ denote the misclassification rate at level $l$ after label alignment within each parent node (so that errors are measured locally within each split). Define the overall tree error by $p_{\text{tree}} := \sum_{l=1}^L p_l$.
This additive metric reflects the cumulative structural distortion across levels.
\end{remark}

If a recursive algorithm achieves local near-optimality $\delta_l$ at each split, then by repeated application of the level-wise stability bound,
\begin{equation}\label{eq:tree_bound}
    p_{\text{tree}} \;\lesssim\; \sum_{l=1}^L \kappa_l\,\delta_l,
\end{equation}
up to benchmark-approximation and displacement terms at each level. Thus, provided $\delta_l$ decays sufficiently fast with depth (and the geometry remains consistent), the overall tree structure is stable.

\subsection{Dynamic clustering: Tracking and drift control}\label{subsec:6.3}

Consider a time-varying setting where the benchmark $(\mathcal{C}^*_t,\theta^*_t)$ evolves over $t=1,\dots,T$.
Let
\[
\eta_t^{\text{drift}} := d_{\text{Hausdorff}}(\theta^*_t,\theta^*_{t-1})
\]
quantify the geometric drift of the truth (permutation-invariant), and let $\eta_t^{\text{alg}}$ be the permutation-matched displacement of the algorithm relative to the \emph{previous} anchors $\theta^*_{t-1}$ (e.g., via warm starts). 
Warm-started algorithms naturally control displacement relative to the previous anchors, since the previous solution provides the initialization for the current step.
By the triangle inequality, the total displacement from the current truth is bounded by
\[
\eta_t \;:=\; \eta_t^{\text{alg}} + \eta_t^{\text{drift}}.
\]
Define the time-$t$ condition number at effective margin $\gamma_t-\eta_t$ by
\begin{equation}\label{eq:kappa_t_def}
\kappa_t(\eta_t) \;:=\; \frac{g(D_t)}{\Delta_g(\gamma_t-\eta_t;D_t)}.
\end{equation}

\begin{proposition}[Tracking stability]\label{prop:tracking}
At time $t$, if the total displacement is small ($\eta_t<\gamma_t$) and $\Delta_g(\gamma_t-\eta_t;D_t)>0$, then a warm-started solution with optimization gap $\delta_t$ satisfies
\begin{equation}\label{eq:tracking_bound}
    p_t \;\lesssim\; \kappa_t(\eta_t)\,\delta_t
    \;+\; \frac{L_g\cdot(\eta_t^{\text{alg}}+\eta_t^{\text{drift}})}{\Delta_g(\gamma_t-\eta_t;D_t)}.
\end{equation}
\end{proposition}

\begin{proof}[Sketch]
    By the triangle inequality, the total displacement of the candidate prototypes from the current benchmark is bounded by the sum of the algorithmic displacement (relative to the previous step) and the geometric drift of the true anchors: $\eta_t \le \eta_t^{\text{alg}} + \eta_t^{\text{drift}}$. Substituting this $\eta_t$ and the time-$t$ geometric parameters into the condition-number bound (Corollary~\ref{cor:kappa_bound}) yields the tracking guarantee. See Appendix~\ref{app:extensions_proofs} for details. 
\end{proof}

\section{Discussion}\label{sec:7}

This paper develops a geometric stability theory for prototype-based clustering that makes precise when \emph{optimization success} can be interpreted as \emph{structural recovery}. The key move is to separate the role of the algorithm from the role of the instance. Algorithms control how close one gets to the optimum, while the data geometry, together with the chosen loss, controls whether near-optimality is informative about the partition. Our analysis crystallizes this separation through a single dimensionless quantity, the clustering condition number, and follows a simple chain in which geometry induces a pointwise cost for misclassification, which aggregates into an excess-loss lower bound, and this, combined with an optimization upper bound, yields nonasymptotic stability guarantees. The refined results sharpen this picture by showing where errors can occur (a boundary belt) and how near-optimal solutions concentrate (a small Hamming tube), culminating in a practical diagnostic procedure.
In short, a small optimization gap together with a favorable condition number
guarantees that the recovered partition must be close to the benchmark clustering,
\[
\text{small optimization gap}
\;+\;
\text{small condition number}
\;\Longrightarrow\;
\text{small clustering error}.
\]

A central takeaway is therefore a \emph{conditioning viewpoint} for clustering. 
Reliability depends not only on the convergence behavior of a solver, but also on whether the underlying instance is well-conditioned under the selected objective.
When the condition number is small, many computational routes to a low objective value are automatically also routes to the correct structure; when it is large, even the global minimizer can be structurally ambiguous, because the loss landscape fails to separate distinct partitions. In this sense, the condition number provides a principled language for interpreting a ubiquitous empirical phenomenon, namely that repeated runs of different algorithms may return different partitions with nearly indistinguishable objective values. Our theory explains when such variability is unavoidable and when it is a symptom of inadequate optimization.

This perspective also suggests a stability-based way to \emph{interrogate modeling choices}. In practice, the number of clusters $k$ is often chosen by a mixture of heuristics and domain knowledge. While our analysis is stated for a fixed $k$, the condition number and the associated certificates provide a natural sanity check, since a plausible model order should yield solutions that are both low-loss and geometrically stable. Empirically, one expects over-refined choices of $k$ to manifest as collapsing separation and rapidly deteriorating stability certificates, whereas under-specified choices tend to produce persistent boundary belts that do not shrink under improved optimization. Formalizing such stability-guided selection rules is an attractive direction, but already at the conceptual level our results clarify that a ``good'' choice of $k$ should make the clustering instance well-conditioned.

A natural direction is to extend the theory beyond strict separation. In many practical datasets, clusters overlap partially, with most points well inside their clusters and a fraction occupying ambiguous boundary regions. The core--belt decomposition in Section~\ref{sec:5} already provides partial results in this direction, certifying exact recovery for interior points even when the global margin is small. A more ambitious program would replace the uniform margin $\gamma$ with a margin profile describing the distribution of point-specific margins, analogous to the Tsybakov margin condition in classification \cite{mammen_smooth_nodate}. Under such a profile, the condition number would depend on quantiles of the margin distribution rather than its minimum, giving bounds that degrade with the fraction of ambiguous points. We view this as the most promising path toward a stability theory that works without clean global separation.

Our results are designed to be usable as a lightweight diagnostic layer on top of standard clustering pipelines. 
A simple diagnostic workflow is as follows.
First, run the chosen algorithm from multiple random initializations (or with multiple solvers) and record the best achieved objective values; the dispersion across restarts provides an empirical proxy for the optimization gap, as in Section~\ref{subsec:6.4}. Second, compute observable geometric proxies from the returned solution (within-cluster radius and prototype separation) and convert them into a guarded condition number estimate. Reporting the resulting certificate alongside the clustering output (for example, the estimated condition number and the implied conservative bound on misclassification mass) provides an interpretable measure of \emph{structural reliability} that is complementary to conventional fit diagnostics such as within-cluster sum of squares. Third, when repeated runs return noticeably different partitions, the Hamming-tube viewpoint suggests treating this variability itself as information. Large pairwise Hamming distances among near-optimal solutions indicate either poor optimization (large gaps across runs) or, more fundamentally, an ill-conditioned geometry. In either case, the diagnostic separates ``try harder optimization'' from ``the instance is intrinsically ambiguous,'' and can motivate follow-up actions such as changing the loss (e.g., Huber tuning), revisiting $k$, or switching to a robust proxy radius when extreme points dominate the max-radius statistic.

Beyond metric prototype clustering, the proof architecture is intentionally modular. The stability argument rests on two ingredients, (i) a notion of margin that turns a structural mistake into a minimum increase in pointwise loss, and (ii) an aggregation step that converts those pointwise increases into a global excess-loss lower bound. Wherever analogous ingredients can be defined, including on graphs, under non-Euclidean dissimilarities, or for other prototype-like objectives, the same condition-number logic can be transplanted. In particular, this opens a path for connecting guarantees obtained in relaxations (e.g., spectral embeddings or convex programs) to statements about the recovered partition, using a common geometric interface rather than algorithm-specific dynamics.

Finally, the results provide conceptual support for downstream analysis that depends on the estimated partition. Many scientific uses of clustering treat the recovered groups as objects of inference (comparing group means, estimating group-level effects, or reporting cluster-level summaries), yet such conclusions are fragile when the partition itself is not stable. By showing that near-optimal solutions concentrate in a small neighborhood of the benchmark under favorable geometry, our theory supplies a structural stability premise under which post-clustering summaries become more reproducible and more interpretable. We hope this encourages a shift in practice. Stability should be viewed as part of the evidentiary standard for using clusters in subsequent scientific claims, on par with conventional optimization diagnostics.

\appendix

\section{Proofs for Section 4 (exact recovery and phase transitions)}\label{app:sec4}

In this appendix, we provide the full proof of the exact recovery thresholds (Theorem~\ref{thm:two_cluster_thresholds}) and present adversarial worst-case constructions to establish the order-wise tightness of these bounds.

\subsection{Proof of Theorem~\ref{thm:two_cluster_thresholds}: Exact recovery thresholds}
\label{app:proof_thm41_full}

Recall the two-ball model in Section~\ref{subsec:4.3}: for $k=2$,
\[
C_1^*\subset B(\theta_1^*,D),\qquad 
C_2^*\subset B(\theta_2^*,D),\qquad 
\|\theta_1^*-\theta_2^*\|=\Delta>2D.
\]
Let $n_j:=|C_j^*|$ with $n_1+n_2=n$, and assume without loss of generality that $n_1\le n_2$.
Thus the balance coefficient is $c_b=n_1/n\in(0,1/2]$.
Define the geometric margin $\gamma:=\Delta-2D>0$.

For part (i) we study $k$-means with squared loss $g(r)=r^2$ and prototypes in $\mathbb R^d$.
For part (ii) we study the linear objective $g(r)=r$ in the \emph{continuous} prototype space, \emph{restricted to the collinear two-ball model}: all sample points are assumed to lie on the one-dimensional axis spanned by $\theta_1^*$ and $\theta_2^*$, so after identifying that axis with $\mathbb R$, part (ii) becomes a one-dimensional $k$-median statement.
The additional discreteness constraint of $k$-medoids is \emph{not} handled here and should be treated separately.

Let $\widehat{\mathcal C}=\{\widehat C_1,\widehat C_2\}$ be an arbitrary candidate partition.
Relabel $\widehat{\mathcal C}$, if necessary, so that it is optimally aligned with $\mathcal C^*$:
\[
|\widehat C_1\cap C_1^*|+|\widehat C_2\cap C_2^*|
\;\ge\;
|\widehat C_1\cap C_2^*|+|\widehat C_2\cap C_1^*|.
\]
Equivalently, if we define
\[
M_1:=C_1^*\cap \widehat C_2,\qquad 
M_2:=C_2^*\cap \widehat C_1,\qquad
m_1:=|M_1|,\quad m_2:=|M_2|,\quad m:=m_1+m_2,
\]
then this relabeling ensures
\begin{equation}\label{eq:thm41_relabeling_app}
m\le \frac{n}{2}.
\end{equation}
Moreover, $\widehat{\mathcal C}\neq \mathcal C^*$ up to label swapping if and only if $m\ge 1$.

\textbf{Benchmark upper bound.}
Because $x_i\in B(\theta_j^*,D)$ for $i\in C_j^*$ and $g$ is nondecreasing,
\begin{equation}\label{eq:thm41_benchUB_app}
\min_{\theta}\mathcal L_n(\mathcal C^*,\theta)
\le \mathcal L_n(\mathcal C^*,\theta^*)
=\sum_{i\in C_1^*} g(\|x_i-\theta_1^*\|)+\sum_{i\in C_2^*} g(\|x_i-\theta_2^*\|)
\le n\,g(D).
\end{equation}


\begin{lemma}[One-dimensional merge penalty for linear loss]\label{lem:1d_merge_penalty_app}
Let
\[
I_1=[u_1-D,u_1+D],\qquad I_2=[u_2-D,u_2+D],\qquad u_2-u_1-2D=\gamma>0.
\]
For a finite set $S\subset\mathbb R$, define the one-dimensional median objective
\[
\phi(S):=\min_{t\in\mathbb R}\sum_{x\in S}|x-t|.
\]
Let $A\subset I_1$ and $B\subset I_2$, with $a:=|A|$ and $b:=|B|$. Then
\begin{equation}\label{eq:1d_merge_penalty}
\phi(A\cup B)\;\ge\;\phi(A)+\phi(B)+\gamma\,\min\{a,b\}.
\end{equation}
\end{lemma}

\begin{proof}
Write
\[
\tau_1:=u_1+D,\qquad \tau_2:=u_2-D,
\]
so that $\tau_2-\tau_1=\gamma$, every point of $A$ lies weakly to the left of $\tau_1$, and every point of $B$ lies weakly to the right of $\tau_2$.

We first treat the case $a\ge b$.
Let
\[
F(t):=\sum_{x\in A}|x-t|+\sum_{y\in B}|y-t|,\qquad t\in\mathbb R.
\]
We claim that $F$ admits a minimizer $t^*$ with $t^*\le \tau_1$.
Indeed, if $t>\tau_1$, set $t':=\tau_1$. Since every $x\in A$ satisfies $x\le \tau_1<t$,
\[
|x-t|=|x-t'|+(t-t').
\]
For every $y\in B$, the triangle inequality gives
\[
|y-t'|\le |y-t|+|t-t'|=|y-t|+(t-t').
\]
Summing these inequalities over $A$ and $B$ yields
\[
F(t')-F(t)\le -(a-b)(t-t')\le 0,
\]
because $a\ge b$. Hence replacing any $t>\tau_1$ by $\tau_1$ does not increase the objective, proving the claim.

Now fix such a minimizer $t^*\le \tau_1$. Then
\[
\phi(A\cup B)=F(t^*)=\sum_{x\in A}|x-t^*|+\sum_{y\in B}|y-t^*|.
\]
The first term is bounded below by $\phi(A)$:
\[
\sum_{x\in A}|x-t^*|\ge \phi(A).
\]
For the second term, since $t^*\le \tau_1<\tau_2\le y$ for every $y\in B$,
\[
|y-t^*|=y-t^*\ge y-\tau_1.
\]
Therefore
\[
\sum_{y\in B}|y-t^*|\ge \sum_{y\in B}(y-\tau_1).
\]
Because $\tau_2\le y$ for all $y\in B$,
\[
\sum_{y\in B}(y-\tau_1)=\sum_{y\in B}(y-\tau_2)+b(\tau_2-\tau_1)
\ge \phi(B)+b\gamma,
\]
where we used $\phi(B)\le \sum_{y\in B}|y-\tau_2|=\sum_{y\in B}(y-\tau_2)$.
Combining the previous displays gives
\[
\phi(A\cup B)\ge \phi(A)+\phi(B)+b\gamma
=\phi(A)+\phi(B)+\gamma\min\{a,b\},
\]
which proves \eqref{eq:1d_merge_penalty} when $a\ge b$.

The case $a\le b$ is symmetric.
Indeed, by the same clipping argument, $F$ admits a minimizer $t^*\ge \tau_2$.
Then
\[
\sum_{y\in B}|y-t^*|\ge \phi(B),
\]
and since every $x\in A$ satisfies $x\le \tau_1<\tau_2\le t^*$,
\[
|x-t^*|=t^*-x\ge \tau_2-x.
\]
Thus
\[
\sum_{x\in A}|x-t^*|\ge \sum_{x\in A}(\tau_2-x)
=\sum_{x\in A}(\tau_1-x)+a(\tau_2-\tau_1)
\ge \phi(A)+a\gamma,
\]
because $\phi(A)\le \sum_{x\in A}|x-\tau_1|=\sum_{x\in A}(\tau_1-x)$.
Hence
\[
\phi(A\cup B)\ge \phi(A)+\phi(B)+a\gamma
=\phi(A)+\phi(B)+\gamma\min\{a,b\}.
\]
This completes the proof.
\end{proof}


\begin{lemma}[Lower bound on the $k$-means mixing coefficient]\label{lem:anova_mix_lb_app}
Assume $n_1\le n_2$ and let $m_1,m_2\ge 0$ satisfy
\[
m_1\le n_1,\qquad m_2\le n_2,\qquad m:=m_1+m_2\le \frac{n}{2}.
\]
Define
\begin{equation}\label{eq:mixcoeff_def_app}
\Psi(m_1,m_2)
:=
\frac{(n_1-m_1)m_2}{\,n_1-m_1+m_2\,}
+
\frac{(n_2-m_2)m_1}{\,n_2-m_2+m_1\,},
\end{equation}
with the convention that a term is interpreted as $0$ when its numerator is $0$.
Then
\begin{equation}\label{eq:mixcoeff_lb_app}
\Psi(m_1,m_2)\;\ge\; \frac{n_1}{n}(m_1+m_2)=c_b\,m.
\end{equation}
\end{lemma}

\begin{proof}
Fix the total number of misclassified points $m=m_1+m_2$ and write $a:=m_1$, so that $m_2=m-a$.
Because $m\le n/2\le n_2$, the admissible range is
\[
a\in[0,\min\{n_1,m\}].
\]
Consider the function
\[
q_m(a)
:=
\frac{(n_1-a)(m-a)}{\,n_1+m-2a\,}
+
\frac{(n_2-m+a)a}{\,n_2-m+2a\,}.
\]
Then $\Psi(m_1,m_2)=q_m(m_1)$.

We first show that $q_m$ is concave on its domain.
For the first term, set
\[
z_1:=n_1+m-2a,\qquad d_1:=n_1-m.
\]
Then
\[
n_1-a=\frac{z_1+d_1}{2},\qquad m-a=\frac{z_1-d_1}{2},
\]
and therefore
\[
\frac{(n_1-a)(m-a)}{n_1+m-2a}
=
\frac{z_1^2-d_1^2}{4z_1}
=
\frac{z_1}{4}-\frac{d_1^2}{4z_1}.
\]
Since $z_1$ is affine in $a$ and $z\mapsto -d_1^2/(4z)$ is concave on $(0,\infty)$, the first term is concave in $a$.

For the second term, set
\[
z_2:=n_2-m+2a,\qquad d_2:=n_2-m.
\]
Then
\[
a=\frac{z_2-d_2}{2},\qquad n_2-m+a=\frac{z_2+d_2}{2},
\]
and hence
\[
\frac{(n_2-m+a)a}{n_2-m+2a}
=
\frac{z_2^2-d_2^2}{4z_2}
=
\frac{z_2}{4}-\frac{d_2^2}{4z_2},
\]
which is also concave in $a$.
Therefore $q_m$ is concave.

A concave function on an interval attains its minimum at an endpoint, so it suffices to check the endpoints.

\smallskip
\noindent\emph{Endpoint 1: $a=0$.}
Then $m_1=0$ and $m_2=m$, so
\[
q_m(0)=\frac{n_1m}{n_1+m}.
\]
Because $m\le n/2\le n_2$, we have $n_1+m\le n_1+n_2=n$, hence
\[
q_m(0)\ge \frac{n_1m}{n}=c_b\,m.
\]

\smallskip
\noindent\emph{Endpoint 2a: $a=m$ (possible when $m\le n_1$).}
Then $m_1=m$ and $m_2=0$, so
\[
q_m(m)=\frac{n_2m}{n_2+m}.
\]
We claim that $q_m(m)\ge n_1m/n$.
Indeed,
\[
\frac{n_2m}{n_2+m}\ge \frac{n_1m}{n}
\quad\Longleftrightarrow\quad
n\,n_2\ge n_1(n_2+m)
\quad\Longleftrightarrow\quad
n_2^2\ge n_1m.
\]
This is true because $m\le n_1\le n_2$.

\smallskip
\noindent\emph{Endpoint 2b: $a=n_1$ (possible when $m>n_1$).}
Then $m_1=n_1$ and $m_2=m-n_1$, so
\[
q_m(n_1)=\frac{n_1(n-m)}{n+n_1-m}.
\]
We claim that $q_m(n_1)\ge n_1m/n$.
Indeed,
\[
\frac{n_1(n-m)}{n+n_1-m}\ge \frac{n_1m}{n}
\quad\Longleftrightarrow\quad
n(n-m)\ge m(n+n_1-m)
\quad\Longleftrightarrow\quad
(n-m)^2\ge n_1m.
\]
Since $m\le n/2$ and $n_1\le n/2$, we have
\[
n-m\ge \frac{n}{2}\ge m
\qquad\text{and}\qquad
n-m\ge \frac{n}{2}\ge n_1,
\]
hence $(n-m)^2\ge n_1m$.

Thus every endpoint value of $q_m$ is at least $n_1m/n$, and therefore
\[
\Psi(m_1,m_2)=q_m(m_1)\ge \frac{n_1m}{n}=c_b\,m.
\]
This proves \eqref{eq:mixcoeff_lb_app}.
\end{proof}

\begin{proof}[Proof of Theorem~\ref{thm:two_cluster_thresholds}]
Let $\widehat{\mathcal C}=\{\widehat C_1,\widehat C_2\}$ be any partition distinct from $\mathcal C^*$ up to label swapping, and keep the relabeling and notation introduced above. Then $m\ge 1$ and \eqref{eq:thm41_relabeling_app} holds.

We show that
\[
\min_{\theta}\mathcal L_n(\widehat{\mathcal C},\theta)
>
\min_{\theta}\mathcal L_n(\mathcal C^*,\theta)
\]
under the stated threshold in each part.

\medskip
\noindent\textbf{Part (i): squared loss ($k$-means), $g(r)=r^2$.}

Set
\[
A_1:=C_1^*\setminus M_1,\qquad B_2:=M_2,\qquad
A_2:=C_2^*\setminus M_2,\qquad B_1:=M_1.
\]
Then
\[
\widehat C_1=A_1\cup B_2,\qquad \widehat C_2=A_2\cup B_1,
\]
with $A_1,B_1\subset B(\theta_1^*,D)$ and $A_2,B_2\subset B(\theta_2^*,D)$.

\emph{\underbar{Step 1}: Lower bound the cross-ball ANOVA penalty.}
For any finite set $S\subset\mathbb R^d$, write
\[
\mathrm{SSE}(S):=\min_{\theta}\sum_{x\in S}\|x-\theta\|^2.
\]
If both $A_1$ and $B_2$ are nonempty, the ANOVA identity gives
\[
\mathrm{SSE}(A_1\cup B_2)
=
\mathrm{SSE}(A_1)+\mathrm{SSE}(B_2)
+
\frac{|A_1||B_2|}{|A_1|+|B_2|}\,
\|\bar A_1-\bar B_2\|^2,
\]
where $\bar A_1$ and $\bar B_2$ denote the sample means.
Since Euclidean balls are convex, $\bar A_1\in B(\theta_1^*,D)$ and $\bar B_2\in B(\theta_2^*,D)$.
Therefore
\[
\|\bar A_1-\bar B_2\|\ge \Delta-2D=\gamma.
\]
Hence
\[
\mathrm{SSE}(A_1\cup B_2)
\ge
\mathrm{SSE}(A_1)+\mathrm{SSE}(B_2)
+
\frac{(n_1-m_1)m_2}{n_1-m_1+m_2}\gamma^2.
\]
By continuity, the same inequality remains valid when one of the two sets is empty, with the convention that the last term is then $0$.
Analogously,
\[
\mathrm{SSE}(A_2\cup B_1)
\ge
\mathrm{SSE}(A_2)+\mathrm{SSE}(B_1)
+
\frac{(n_2-m_2)m_1}{n_2-m_2+m_1}\gamma^2.
\]
Summing the two bounds and dropping the nonnegative terms $\mathrm{SSE}(B_1)$ and $\mathrm{SSE}(B_2)$ yields
\begin{equation}\label{eq:penalty_app_new}
\min_{\theta}\mathcal L_n(\widehat{\mathcal C},\theta)
\ge
\mathrm{SSE}(A_1)+\mathrm{SSE}(A_2)+\Psi(m_1,m_2)\gamma^2,
\end{equation}
where $\Psi(m_1,m_2)$ is defined in \eqref{eq:mixcoeff_def_app}.

\emph{\underbar{Step 2}: Upper bound the gain from deleting misclassified points from the benchmark clusters.}
Fix any set $S\subset B(\theta^*,D)$ and any subset $T\subset S$.
Let $\mu_{S\setminus T}$ denote the mean of $S\setminus T$.
Because $B(\theta^*,D)$ is convex, $\mu_{S\setminus T}\in B(\theta^*,D)$.
Hence for every $x\in T$,
\[
\|x-\mu_{S\setminus T}\|
\le \|x-\theta^*\|+\|\mu_{S\setminus T}-\theta^*\|
\le D+D=2D.
\]
Using $\mu_{S\setminus T}$ as a feasible center for $S$ gives
\[
\mathrm{SSE}(S)-\mathrm{SSE}(S\setminus T)
\le \sum_{x\in T}\|x-\mu_{S\setminus T}\|^2
\le 4|T|D^2.
\]
Applying this with $(S,T)=(C_1^*,M_1)$ and $(C_2^*,M_2)$ yields
\begin{equation}\label{eq:gain_bound_app_new}
\bigl[\mathrm{SSE}(C_1^*)-\mathrm{SSE}(A_1)\bigr]
+
\bigl[\mathrm{SSE}(C_2^*)-\mathrm{SSE}(A_2)\bigr]
\le 4mD^2.
\end{equation}

\emph{\underbar{Step 3}: Compare with the benchmark partition.}
Since
\[
\min_{\theta}\mathcal L_n(\mathcal C^*,\theta)
=
\mathrm{SSE}(C_1^*)+\mathrm{SSE}(C_2^*),
\]
combining \eqref{eq:penalty_app_new}, \eqref{eq:gain_bound_app_new}, and Lemma~\ref{lem:anova_mix_lb_app} gives
\begin{align*}
\min_{\theta}\mathcal L_n(\widehat{\mathcal C},\theta)
-
\min_{\theta}\mathcal L_n(\mathcal C^*,\theta)
&\ge
-4mD^2+\Psi(m_1,m_2)\gamma^2 \\
&\ge
-4mD^2+c_bm\gamma^2 \\
&=
m\bigl(c_b\gamma^2-4D^2\bigr).
\end{align*}
If $c_b\gamma^2>4D^2$, the right-hand side is strictly positive because $m\ge 1$.
This proves part (i).

\medskip
\noindent\textbf{Part (ii): linear loss (one-dimensional continuous $k$-median), $g(r)=r$.}

Identify the axis joining $\theta_1^*$ and $\theta_2^*$ with $\mathbb R$.
Write
\[
C_1^*\subset I_1=[u_1-D,u_1+D],\qquad
C_2^*\subset I_2=[u_2-D,u_2+D],\qquad
u_2-u_1-2D=\gamma.
\]
For any finite set $S\subset\mathbb R$, define
\[
\phi(S):=\min_{t\in\mathbb R}\sum_{x\in S}|x-t|.
\]
Then
\[
\min_{\theta}\mathcal L_n(\widehat{\mathcal C},\theta)=\phi(\widehat C_1)+\phi(\widehat C_2),
\qquad
\min_{\theta}\mathcal L_n(\mathcal C^*,\theta)=\phi(C_1^*)+\phi(C_2^*).
\]

\emph{\underbar{Step 1}: Bound the gain from deleting misclassified points from a one-dimensional ball.}
Let $S\subset [u-D,u+D]$ and $T\subset S$, and write $m_T:=|T|$.
Let $t_{S\setminus T}$ be any median of $S\setminus T$.
Since $S\setminus T\subset [u-D,u+D]$ and the interval is convex, every median of $S\setminus T$ lies in $[u-D,u+D]$.
Hence for every $x\in T$,
\[
|x-t_{S\setminus T}|\le 2D.
\]
Using $t_{S\setminus T}$ as a feasible location for $S$,
\[
\phi(S)\le \sum_{x\in S}|x-t_{S\setminus T}|
=\phi(S\setminus T)+\sum_{x\in T}|x-t_{S\setminus T}|
\le \phi(S\setminus T)+2Dm_T.
\]
Therefore
\begin{equation}\label{eq:linear_delete_gain_app_new}
\phi(S\setminus T)\ge \phi(S)-2Dm_T.
\end{equation}
Applying \eqref{eq:linear_delete_gain_app_new} with $(S,T)=(C_1^*,M_1)$ and $(C_2^*,M_2)$ gives
\begin{equation}\label{eq:linear_reduced_sets_app_new}
\phi(A_1)+\phi(A_2)
\ge \phi(C_1^*)+\phi(C_2^*)-2D(m_1+m_2)
= \min_{\theta}\mathcal L_n(\mathcal C^*,\theta)-2mD.
\end{equation}

\emph{\underbar{Step 2}: Lower bound the merge cost of the two mixed clusters.}
By Lemma~\ref{lem:1d_merge_penalty_app},
\[
\phi(\widehat C_1)=\phi(A_1\cup B_2)
\ge \phi(A_1)+\phi(B_2)+\gamma\,\min\{n_1-m_1,m_2\},
\]
and
\[
\phi(\widehat C_2)=\phi(A_2\cup B_1)
\ge \phi(A_2)+\phi(B_1)+\gamma\,\min\{n_2-m_2,m_1\}.
\]
Since $\phi(B_1)\ge 0$ and $\phi(B_2)\ge 0$, summing yields
\begin{equation}\label{eq:linear_merge_lb_app_new}
\min_{\theta}\mathcal L_n(\widehat{\mathcal C},\theta)
\ge \phi(A_1)+\phi(A_2)+\gamma P,
\end{equation}
where
\[
P:=\min\{n_1-m_1,m_2\}+\min\{n_2-m_2,m_1\}.
\]

\emph{\underbar{Step 3}: Compare with the benchmark partition.}
Combining \eqref{eq:linear_reduced_sets_app_new} and \eqref{eq:linear_merge_lb_app_new}, we obtain
\begin{equation}\label{eq:linear_diff_precase_app_new}
\min_{\theta}\mathcal L_n(\widehat{\mathcal C},\theta)
-\min_{\theta}\mathcal L_n(\mathcal C^*,\theta)
\ge \gamma P - 2mD.
\end{equation}

We now lower bound $P$ using only \eqref{eq:thm41_relabeling_app}.
Since $m=m_1+m_2\le n/2$ and $n_2\ge n/2$, we have
\[
n_2-m_2\ge \frac{n}{2}-m_2\ge m_1,
\]
and therefore
\[
\min\{n_2-m_2,m_1\}=m_1.
\]
Hence
\[
P=m_1+\min\{n_1-m_1,m_2\}.
\]

\smallskip
\noindent\emph{Case 1: $m_2\le n_1-m_1$.}
Then $P=m_1+m_2=m$, and \eqref{eq:linear_diff_precase_app_new} becomes
\[
\min_{\theta}\mathcal L_n(\widehat{\mathcal C},\theta)
-\min_{\theta}\mathcal L_n(\mathcal C^*,\theta)
\ge m(\gamma-2D).
\]
Now $c_b\le 1/2$, so the condition $c_b\gamma>D$ implies
\[
\gamma>\frac{D}{c_b}\ge 2D.
\]
Since $m\ge 1$, the right-hand side is strictly positive.

\smallskip
\noindent\emph{Case 2: $m_2>n_1-m_1$.}
Then $P=m_1+(n_1-m_1)=n_1$.
Using again \eqref{eq:thm41_relabeling_app}, we have $m\le n/2$, so \eqref{eq:linear_diff_precase_app_new} gives
\[
\min_{\theta}\mathcal L_n(\widehat{\mathcal C},\theta)
-\min_{\theta}\mathcal L_n(\mathcal C^*,\theta)
\ge n_1\gamma - 2mD
\ge n_1\gamma - nD
= n(c_b\gamma-D),
\]
which is strictly positive by the assumed condition $c_b\gamma>D$.

Thus in both cases,
\[
\min_{\theta}\mathcal L_n(\widehat{\mathcal C},\theta)
>
\min_{\theta}\mathcal L_n(\mathcal C^*,\theta),
\]
proving part (ii).

Since $\widehat{\mathcal C}\neq \mathcal C^*$ up to label swapping was arbitrary, we conclude that under the sufficient conditions
\[
c_b\gamma^2>4D^2
\qquad\text{for $k$-means,}
\]
and
\[
c_b\gamma>D
\qquad\text{for one-dimensional continuous $k$-median,}
\]
every partition different from $\mathcal C^*$ up to label swapping is strictly suboptimal for the profiled objective
\[
\mathcal C\mapsto \min_{\theta}\mathcal L_n(\mathcal C,\theta).
\]
Hence $\mathcal C^*$ is the unique minimizer of the profiled objective, up to label swapping.

Finally, substituting $\gamma=\Delta-2D$ gives the explicit thresholds:
\[
c_b(\Delta-2D)^2>4D^2
\quad\Longleftrightarrow\quad
\frac{\Delta}{D}>2+\frac{2}{\sqrt{c_b}}
\qquad (k\text{-means}),
\]
and
\[
c_b(\Delta-2D)>D
\quad\Longleftrightarrow\quad
\frac{\Delta}{D}>2+\frac{1}{c_b}
\qquad (\text{one-dimensional continuous }k\text{-median}).
\]
As shown in Appendix~\ref{app:tightness}, the scalings $1/\sqrt{c_b}$ and $1/c_b$ are order-wise tight.
\end{proof}
\subsection{Tightness and worst-case constructions}\label{app:tightness}

In this part, we prove the tightness of the exact recovery thresholds presented in Theorem~\ref{thm:two_cluster_thresholds}. We construct specific "worst-case" configurations where the separation $\Delta$ is slightly below the stated thresholds, and show that the global minimizer of the objective function fails to recover the benchmark partition.

The failure mode we exploit is ``Heavy Cluster Splitting'': when one cluster is significantly more massive than the other ($n_2 \gg n_1$), the algorithm may reduce the total loss by placing \emph{both} prototypes within the massive cluster (to reduce its within-cluster variance/dispersion) while treating the small cluster as outliers assigned to the nearest prototype.

\subsubsection{Setup: The adversarial 1D configuration}

Consider a one-dimensional dataset $\mathcal{X} = \mathbb{R}$ with $k=2$ clusters. Let $n_1$ and $n_2$ denote the sizes of the two benchmark clusters $C_1^*$ and $C_2^*$, with $n_2 > n_1$. The balance coefficient is $c_b = n_1/(n_1+n_2)$. We assume the severe imbalance regime where $n_2 \gg n_1$.

We construct the geometry as follows:
\begin{itemize}
    \item {Cluster 1 (The Victim):} A point mass of $n_1$ points located at $x=0$.
    \item {Cluster 2 (The Heavy Target):} Two point masses, each of size $n_2/2$, located at $x = \Delta - D$ and $x = \Delta + D$.
\end{itemize}
The benchmark parameters for this configuration are:
\begin{itemize}
    \item {Anchors:} $\theta_1^* = 0$, $\theta_2^* = \Delta$.
    \item {Radius:} The maximum distance from $\theta_2^*$ to its points is $D$. Thus $D_{\text{eff}} = D$.
    \item {Separation:} $d(\theta_1^*, \theta_2^*) = \Delta$.
\end{itemize}
We analyze whether the "Correct" solution (recovering $C_1^*, C_2^*$) is preferred over the "Splitting" solution (splitting $C_2^*$ and merging $C_1^*$).

\subsubsection{Tightness for $k$-means (squared loss)}

Under the $k$-means objective $\mathcal{L}(\theta) = \sum_{x \in S} \min_j \|x - \theta_j\|^2$.

\textit{Cost of the correct solution ($\mathcal{L}_{\text{correct}}$).}
The solver places prototypes at the centroids of $C_1^*$ and $C_2^*$.
\begin{itemize}
    \item $\hat{\theta}_1 = 0$. Contribution from $C_1^*$: $0$.
    \item $\hat{\theta}_2 = \Delta$. Contribution from $C_2^*$: Since $C_2^*$ has points at $\Delta \pm D$, the cost is $\frac{n_2}{2}(-D)^2 + \frac{n_2}{2}(D)^2 = n_2 D^2$.
\end{itemize}
\[
\mathcal{L}_{\text{correct}} = n_2 D^2.
\]

\textit{Cost of the splitting solution ($\mathcal{L}_{\text{split}}$).}
Consider a candidate solution that places prototypes at the two sub-modes of the heavy cluster: $\hat{\theta} = (\Delta - D, \Delta + D)$.
\begin{itemize}
    \item {Cluster 2 points:} The $n_2/2$ points at $\Delta-D$ are assigned to $\hat{\theta}_1=\Delta-D$ (cost 0). The $n_2/2$ points at $\Delta+D$ are assigned to $\hat{\theta}_2=\Delta+D$ (cost 0). Total contribution from $C_2^*$: $0$.
    \item {Cluster 1 points:} The $n_1$ points at $0$ are assigned to the nearest prototype $\hat{\theta}_1 = \Delta - D$. The squared distance is $(\Delta - D - 0)^2$.
\end{itemize}
\[
\mathcal{L}_{\text{split}} = n_1 (\Delta - D)^2.
\]

\textit{The failure condition.}
Exact recovery fails if $\mathcal{L}_{\text{split}} < \mathcal{L}_{\text{correct}}$, i.e.,
\[
n_1 (\Delta - D)^2 < n_2 D^2.
\]
Taking the square root:
\[
\sqrt{n_1} (\Delta - D) < \sqrt{n_2} D \implies \frac{\Delta}{D} - 1 < \sqrt{\frac{n_2}{n_1}}.
\]
Rearranging for the ratio $\Delta/D$:
\[
\frac{\Delta}{D} < 1 + \sqrt{\frac{n_2}{n_1}}.
\]
Recall that $c_b \approx n_1/n_2$ for small $c_b$. Thus $\sqrt{n_2/n_1} \approx 1/\sqrt{c_b}$. The failure condition becomes:
\[
\frac{\Delta}{D} < 1 + \frac{1}{\sqrt{c_b}}.
\]
This confirms that a separation scaling of order $1/\sqrt{c_b}$ is necessary. If $\Delta/D$ falls below this order, the objective is minimized by splitting the heavy cluster rather than separating the distant small cluster.

\subsubsection{Tightness for $k$-medoids (linear loss)}

Under the $k$-medoids objective $\mathcal{L}(\theta) = \sum_{x \in S} \min_j \|x - \theta_j\|$ (assuming continuous prototypes for simplicity, or selecting points from the dataset which yields the same result here).

\emph{Cost of the correct solution ($\mathcal{L}_{\text{correct}}$).}
Prototypes are placed at the medians.
\begin{itemize}
    \item $\hat{\theta}_1 = 0$. Contribution from $C_1^*$: $0$.
    \item $\hat{\theta}_2 = \Delta$. Contribution from $C_2^*$: The points at $\Delta \pm D$ are at distance $D$ from the center $\Delta$. Cost: $\frac{n_2}{2}(D) + \frac{n_2}{2}(D) = n_2 D$.
\end{itemize}
\[
\mathcal{L}_{\text{correct}} = n_2 D.
\]

\emph{2. Cost of the splitting solution ($\mathcal{L}_{\text{split}}$).}
Prototypes are placed at the sub-modes: $\hat{\theta} = (\Delta - D, \Delta + D)$.
\begin{itemize}
    \item {Cluster 2 points:} All points in $C_2^*$ have distance 0 to their assigned prototypes. Cost: $0$.
    \item {Cluster 1 points:} Points at $0$ are assigned to $\hat{\theta}_1 = \Delta - D$. Distance is $|\Delta - D|$.
\end{itemize}
\[
\mathcal{L}_{\text{split}} = n_1 (\Delta - D).
\]

\emph{3. The failure condition.}
Exact recovery fails if $\mathcal{L}_{\text{split}} < \mathcal{L}_{\text{correct}}$, i.e.,
\[
n_1 (\Delta - D) < n_2 D.
\]
Dividing by $n_1 D$:
\[
\frac{\Delta}{D} - 1 < \frac{n_2}{n_1}.
\]
Using $c_b \approx n_1/n_2$, this implies failure when:
\[
\frac{\Delta}{D} < 1 + \frac{1}{c_b}.
\]
This confirms that for $k$-medoids, the separation must scale linearly with the imbalance ratio $1/c_b$ to prevent the heavy cluster from monopolizing the prototypes.

\subsubsection{Conclusion on sharpness}

Comparing the two failure conditions:
\begin{itemize}
    \item {$k$-means Failure:} $\Delta/D \lesssim \sqrt{1/c_b}$.
    \item {$k$-medoids Failure:} $\Delta/D \lesssim 1/c_b$.
\end{itemize}
Since $c_b \ll 1$, we have $1/c_b \gg \sqrt{1/c_b}$. This proves that the thresholds derived in Theorem~\ref{thm:two_cluster_thresholds} are tight and reflect a fundamental geometric reality: the linear loss of $k$-medoids provides less "force" to hold a distant small cluster against the "gain" of splitting a dense heavy cluster, thereby requiring larger separation to ensure stability in the face of imbalance.

\section{Proofs for Section 5 (refined stability and local geometry)}\label{app:refined_proofs}

In this appendix, we provide detailed proofs for the local parameter stability (Proposition~\ref{prop:eta_control}) and the geometry of the near-optimal solution set (Theorem~\ref{thm:tube}).

\subsection{Notation, alignment, and label functions}

Let $S=\{x_1,\dots,x_n\}\subset \mathbb R^d$. A partition $\mathcal C=\{C_1,\dots,C_k\}$ is represented by a label function
\[
c:[n]\to [k],\qquad c(i)=j \iff i\in C_j.
\]
Likewise, the benchmark partition $\mathcal C^*=\{C_1^*,\dots,C_k^*\}$ is represented by $c^*$.

Given two labelings $c,c'$, the (raw) Hamming distance is
\[
d_{\mathrm{Ham}}^{\mathrm{raw}}(c,c'):=\frac1n\sum_{i=1}^n \mathbf 1\{c(i)\neq c'(i)\}.
\]
Because partitions are label-invariant, we define the label-invariant Hamming distance by optimizing over permutations $\Pi_k$:
\begin{equation}\label{eq:ham_label_def}
d_{\mathrm{Ham}}(\mathcal C,\mathcal C')
:=\min_{\pi\in\Pi_k}\frac1n\sum_{i=1}^n \mathbf 1\{c(i)\neq \pi(c'(i))\}.
\end{equation}
It is equivalent to the overlap form in (33) of the main text:
\[
d_{\mathrm{Ham}}(\mathcal C,\mathcal C')
=1-\max_{\pi\in\Pi_k}\frac1n\sum_{j=1}^k |C_j\cap C'_{\pi(j)}|.
\]

Let $\theta^*=(\theta_1^*,\dots,\theta_k^*)$ be the benchmark prototypes and $\hat\theta=(\hat\theta_1,\dots,\hat\theta_k)$ be candidate prototypes.
Define the displacement
\[
\eta(\hat\theta,\theta^*)
:=\min_{\pi\in\Pi_k}\max_{j\in[k]}\|\hat\theta_j-\theta^*_{\pi(j)}\|.
\]
Fix (once and for all in Proposition~\ref{prop:eta_control}) an optimal permutation $\pi^\star$ attaining the above minimum, and relabel the candidate prototypes by $\hat\theta_j\leftarrow \hat\theta_{\pi^\star(j)}$.
After this relabeling, we may write simply
\[
\eta=\max_{j\in[k]}\|\hat\theta_j-\theta_j^*\|.
\]
(Throughout, all comparisons between $\hat C_j$ and $C_j^*$ are understood under the same optimal matching convention.)

\subsection{Useful lemmas}

\begin{lemma}[Triangle inequality for $d_{\mathrm{Ham}}$]\label{lem:ham_triangle}
For any three partitions $\mathcal A,\mathcal B,\mathcal D$,
\[
d_{\mathrm{Ham}}(\mathcal A,\mathcal B)\le d_{\mathrm{Ham}}(\mathcal A,\mathcal D)+d_{\mathrm{Ham}}(\mathcal D,\mathcal B).
\]
\end{lemma}

\begin{proof}
Let $a,d,b:[n]\to[k]$ be label functions for $\mathcal A,\mathcal D,\mathcal B$.

Let $\pi_{AD}\in\Pi_k$ attain the minimum in \eqref{eq:ham_label_def} for $(\mathcal A,\mathcal D)$, and let $\pi_{DB}\in\Pi_k$ attain the minimum for $(\mathcal D,\mathcal B)$.
Define the composed permutation
\[
\pi_{AB}:=\pi_{AD}\circ \pi_{DB}\in\Pi_k.
\]
Then, by definition of the minimum,
\[
d_{\mathrm{Ham}}(\mathcal A,\mathcal B)
=\min_{\pi\in\Pi_k}\frac1n\sum_{i=1}^n \mathbf 1\{a(i)\neq \pi(b(i))\}
\le \frac1n\sum_{i=1}^n \mathbf 1\{a(i)\neq \pi_{AB}(b(i))\}.
\]
Now use the pointwise implication (a union bound for indicators): for any labels $u,v,w$,
\[
\mathbf 1\{u\neq w\}\le \mathbf 1\{u\neq v\}+\mathbf 1\{v\neq w\}.
\]
Apply it with $u=a(i)$, $v=\pi_{AD}(d(i))$, $w=\pi_{AD}(\pi_{DB}(b(i)))=\pi_{AB}(b(i))$:
\[
\mathbf 1\{a(i)\neq \pi_{AB}(b(i))\}
\le \mathbf 1\{a(i)\neq \pi_{AD}(d(i))\}+\mathbf 1\{\pi_{AD}(d(i))\neq \pi_{AD}(\pi_{DB}(b(i)))\}.
\]
Because $\pi_{AD}$ is a bijection on $[k]$, the second indicator is equal to
\[
\mathbf 1\{d(i)\neq \pi_{DB}(b(i))\}.
\]
Summing over $i$ and dividing by $n$ yields
\[
\frac1n\sum_{i=1}^n \mathbf 1\{a(i)\neq \pi_{AB}(b(i))\}
\le
\frac1n\sum_{i=1}^n \mathbf 1\{a(i)\neq \pi_{AD}(d(i))\}
+
\frac1n\sum_{i=1}^n \mathbf 1\{d(i)\neq \pi_{DB}(b(i))\}.
\]
The first term equals $d_{\mathrm{Ham}}(\mathcal A,\mathcal D)$ by optimality of $\pi_{AD}$, and the second equals $d_{\mathrm{Ham}}(\mathcal D,\mathcal B)$ by optimality of $\pi_{DB}$.
This proves the triangle inequality.
\end{proof}

\begin{lemma}[Small displacement implies benchmark assignment is optimal]\label{lem:small_eta_assignment}
Assume the separable regime $\gamma=\Delta_0-2D_{\mathrm{eff}}>0$ and the benchmark margin property (6) in the main text:
for any $i\in C_j^*$ and any $\ell\neq j$,
\[
d(x_i,\theta_\ell^*)\ge d(x_i,\theta_j^*)+\gamma.
\]
If a candidate prototype tuple $\hat\theta$ satisfies $\eta<\gamma/2$, then for every $i\in C_j^*$ and every $\ell\neq j$,
\[
d(x_i,\hat\theta_\ell)>d(x_i,\hat\theta_j),
\]
hence the Voronoi (nearest-prototype) assignment induced by $\hat\theta$ coincides with the benchmark partition $\mathcal C^*$.
In particular,
\[
\min_{\mathcal C}\mathcal L_n(\mathcal C,\hat\theta)=\mathcal L_n(\mathcal C^*,\hat\theta).
\]
\end{lemma}

\begin{proof}
Fix $i\in C_j^*$ and $\ell\neq j$.
By the triangle inequality and the definition of $\eta$ (after alignment),
\[
d(x_i,\hat\theta_\ell)\ge d(x_i,\theta_\ell^*)-\|\hat\theta_\ell-\theta_\ell^*\|\ge d(x_i,\theta_\ell^*)-\eta.
\]
By the benchmark margin property, $d(x_i,\theta_\ell^*)\ge d(x_i,\theta_j^*)+\gamma$, hence
\[
d(x_i,\hat\theta_\ell)\ge d(x_i,\theta_j^*)+\gamma-\eta.
\]
On the other hand, again by the triangle inequality,
\[
d(x_i,\hat\theta_j)\le d(x_i,\theta_j^*)+\|\hat\theta_j-\theta_j^*\|\le d(x_i,\theta_j^*)+\eta.
\]
Combining the two displays gives
\[
d(x_i,\hat\theta_\ell)-d(x_i,\hat\theta_j)\ge (d(x_i,\theta_j^*)+\gamma-\eta)-(d(x_i,\theta_j^*)+\eta)=\gamma-2\eta.
\]
If $\eta<\gamma/2$, then $\gamma-2\eta>0$, so $d(x_i,\hat\theta_\ell)>d(x_i,\hat\theta_j)$ for all $\ell\neq j$.
Therefore each point $i\in C_j^*$ is uniquely closest to $\hat\theta_j$, and the nearest-prototype partition induced by $\hat\theta$ equals $\mathcal C^*$.
The final equality $\min_{\mathcal C}\mathcal L_n(\mathcal C,\hat\theta)=\mathcal L_n(\mathcal C^*,\hat\theta)$ follows immediately.
\end{proof}

\subsection{Proof of Proposition~\ref{prop:core_stability}: Local stability bound}

Let $(\hat{\mathcal C},\hat\theta)$ satisfy the conditions of Theorem~\ref{thm:main_stability}.
Throughout, all label comparisons are understood under the same optimal matching permutation used to define the displacement
\[
\eta=\min_{\pi\in\Pi_k}\max_{j\in[k]}\|\hat\theta_j-\theta^*_{\pi(j)}\| ,
\qquad\text{so after relabeling we may write }\quad
\eta=\max_{j\in[k]}\|\hat\theta_j-\theta_j^*\|.
\]

\emph{Core set and decomposition of misclassified points.}
Fix $s\in[0,D_{\mathrm{eff}})$ and define the (depth-$s$) core index set
\[
I_{\mathrm{core}}(s):=\Bigl\{\,i\in[n]: \|x_i-\theta^*_{c^*(i)}\|\le D_{\mathrm{eff}}-s\,\Bigr\}.
\]
Let $\hat c$ and $c^*$ denote the label functions of $\hat{\mathcal C}$ and $\mathcal C^*$, respectively, and define the global
misclassification set
\[
M:=\{\,i\in[n]:\hat c(i)\neq c^*(i)\,\}.
\]
Split $M$ into disjoint subsets
\[
M_{\mathrm{core}}:=M\cap I_{\mathrm{core}}(s),\qquad
M_{\mathrm{belt}}:=M\setminus M_{\mathrm{core}}.
\]
By definition,
\[
p_{\mathrm{core}}(s)=\frac{|M_{\mathrm{core}}|}{n}.
\]

\emph{\underbar{Step 1}: Strengthened benchmark margin on the $s$-core.}
Recall $\Delta_0:=\min_{j\neq\ell}\|\theta_j^*-\theta_\ell^*\|$ and $\gamma:=\Delta_0-2D_{\mathrm{eff}}>0$.
Fix $i\in I_{\mathrm{core}}(s)$ and write $j=c^*(i)$. For any $\ell\neq j$, the triangle inequality gives
\[
\|x_i-\theta_\ell^*\|
\ge \|\theta_\ell^*-\theta_j^*\|-\|x_i-\theta_j^*\|
\ge \Delta_0-(D_{\mathrm{eff}}-s)
=(2D_{\mathrm{eff}}+\gamma)-(D_{\mathrm{eff}}-s)
= D_{\mathrm{eff}}+\gamma+s.
\]
Since $\|x_i-\theta_j^*\|\le D_{\mathrm{eff}}-s$ for $i\in I_{\mathrm{core}}(s)$, we obtain the core gap
\begin{equation}\label{eq:app_core_gap}
\|x_i-\theta_\ell^*\|-\|x_i-\theta_j^*\|
\ge (D_{\mathrm{eff}}+\gamma+s)-(D_{\mathrm{eff}}-s)
=\gamma+2s.
\end{equation}

\emph{\underbar{Step 2}: Pointwise loss increment for misclassified core points.}
Fix $i\in M_{\mathrm{core}}$ and write $j=c^*(i)$ and $\ell=\hat c(i)\neq j$.
By the triangle inequality and the definition of $\eta$,
\[
\|x_i-\hat\theta_\ell\|
\ge \|x_i-\theta_\ell^*\|-\|\hat\theta_\ell-\theta_\ell^*\|
\ge \|x_i-\theta_\ell^*\|-\eta.
\]
Combine this with \eqref{eq:app_core_gap} to obtain
\[
\|x_i-\hat\theta_\ell\|
\ge \|x_i-\theta_j^*\| + (\gamma+2s) - \eta
= \|x_i-\theta_j^*\| + (\gamma-\eta+2s).
\]
Since $i\in I_{\mathrm{core}}(s)$ implies $\|x_i-\theta_j^*\|\le D_{\mathrm{eff}}-s$, and $g$ is nondecreasing
(Assumption~\ref{def:admissible_loss} in the main text), we may invoke the definition
\[
\Delta_g(\alpha;D):=\inf_{0\le r\le D}\{g(r+\alpha)-g(r)\}
\]
to conclude that for each $i\in M_{\mathrm{core}}$,
\begin{equation}\label{eq:app_core_increment}
g(\|x_i-\hat\theta_{\hat c(i)}\|)-g(\|x_i-\theta^*_{c^*(i)}\|)
\;\ge\;
\Delta_g(\gamma-\eta+2s\,;\,D_{\mathrm{eff}}-s).
\end{equation}

\emph{\underbar{Step 3}: Belt misclassifications are nonnegative and can be dropped.}
For $i\in M_{\mathrm{belt}}$, we still have the global benchmark separation (margin) at level $\gamma$ and radius $D_{\mathrm{eff}}$.
By the same argument as in the proof of Theorem~\ref{thm:main_stability} (specializing \eqref{eq:app_core_increment} to $s=0$),
\[
g(\|x_i-\hat\theta_{\hat c(i)}\|)-g(\|x_i-\theta^*_{c^*(i)}\|)
\;\ge\;
\Delta_g(\gamma-\eta\,;\,D_{\mathrm{eff}}).
\]
Under the separable regime $\gamma-\eta>0$ and the monotonicity of $g$, we have
$\Delta_g(\gamma-\eta;D_{\mathrm{eff}})\ge 0$. Hence
\begin{equation}\label{eq:app_drop_belt}
\sum_{i\in M_{\mathrm{belt}}}\Bigl[g(\|x_i-\hat\theta_{\hat c(i)}\|)-g(\|x_i-\theta^*_{c^*(i)}\|)\Bigr]\ \ge\ 0,
\end{equation}
and we may conservatively drop this term in a lower bound.

\emph{\underbar{Step 4}: Correctly classified points contribute at least $-L_g\eta$.}
Fix $i\notin M$, so $\hat c(i)=c^*(i)=j$. Using that the distance map is $1$-Lipschitz in its second argument and that
$g$ is $L_g$-Lipschitz on the relevant domain, we have
\[
\bigl|\|x_i-\hat\theta_j\|-\|x_i-\theta_j^*\|\bigr|\le \|\hat\theta_j-\theta_j^*\|\le \eta
\quad\Longrightarrow\quad
g(\|x_i-\hat\theta_j\|)-g(\|x_i-\theta_j^*\|)\ge -L_g\eta.
\]
Therefore
\begin{equation}\label{eq:app_correct_lower}
\sum_{i\notin M}\Bigl[g(\|x_i-\hat\theta_{\hat c(i)}\|)-g(\|x_i-\theta^*_{c^*(i)}\|)\Bigr]\ \ge\ -(n-|M|)L_g\eta\ \ge\ -nL_g\eta.
\end{equation}

\emph{\underbar{Step 5}: Aggregate the global excess loss and conclude.}
Decompose the global excess loss relative to the benchmark prototypes:
\begin{align*}
\mathcal L_n(\hat{\mathcal C},\hat\theta)-\mathcal L_n(\mathcal C^*,\theta^*)
&=
\sum_{i=1}^n \Bigl[g(\|x_i-\hat\theta_{\hat c(i)}\|)-g(\|x_i-\theta^*_{c^*(i)}\|)\Bigr]\\
&=
\sum_{i\in M_{\mathrm{core}}}(\cdots)
+\sum_{i\in M_{\mathrm{belt}}}(\cdots)
+\sum_{i\notin M}(\cdots).
\end{align*}
Apply \eqref{eq:app_core_increment} on $M_{\mathrm{core}}$, drop the nonnegative belt contribution using \eqref{eq:app_drop_belt},
and lower bound the correct contribution using \eqref{eq:app_correct_lower}:
\[
\mathcal L_n(\hat{\mathcal C},\hat\theta)-\mathcal L_n(\mathcal C^*,\theta^*)
\ge |M_{\mathrm{core}}|\cdot \Delta_g(\gamma-\eta+2s\,;\,D_{\mathrm{eff}}-s)-nL_g\eta.
\]
Since $|M_{\mathrm{core}}|=np_{\mathrm{core}}(s)$, this becomes
\begin{equation}\label{eq:app_core_excess_lower}
\mathcal L_n(\hat{\mathcal C},\hat\theta)-\mathcal L_n(\mathcal C^*,\theta^*)
\ge n\,p_{\mathrm{core}}(s)\,\Delta_g(\gamma-\eta+2s\,;\,D_{\mathrm{eff}}-s)-nL_g\eta.
\end{equation}

On the other hand, by $(1+\delta)$-near-optimality of $(\hat{\mathcal C},\hat\theta)$ and the definition of $\delta_{\mathrm{approx}}$,
\[
\mathcal L_n(\hat{\mathcal C},\hat\theta)\le (1+\delta)\mathsf{OPT}_n,
\qquad
\mathcal L_n(\mathcal C^*,\theta^*)=(1+\delta_{\mathrm{approx}})\mathsf{OPT}_n,
\]
so
\begin{equation}\label{eq:app_core_excess_upper}
\mathcal L_n(\hat{\mathcal C},\hat\theta)-\mathcal L_n(\mathcal C^*,\theta^*)
\le (\delta+\delta_{\mathrm{approx}})\mathsf{OPT}_n.
\end{equation}
Combine \eqref{eq:app_core_excess_lower} and \eqref{eq:app_core_excess_upper}:
\[
n\,p_{\mathrm{core}}(s)\,\Delta_g(\gamma-\eta+2s\,;\,D_{\mathrm{eff}}-s)-nL_g\eta
\le (\delta+\delta_{\mathrm{approx}})\mathsf{OPT}_n.
\]
Divide by $n\,\Delta_g(\gamma-\eta+2s\,;\,D_{\mathrm{eff}}-s)$ and rearrange to obtain
\[
p_{\mathrm{core}}(s)
\le
\frac{\mathsf{OPT}_n}{n\,\Delta_g(\gamma-\eta+2s\,;\,D_{\mathrm{eff}}-s)}(\delta+\delta_{\mathrm{approx}})
+\frac{L_g\eta}{\Delta_g(\gamma-\eta+2s\,;\,D_{\mathrm{eff}}-s)},
\]
which is exactly \eqref{eq:core_bound}.

\subsection{Proof of Proposition~\ref{prop:eta_control}: Control of displacement}

We prove parts (i) and (ii) separately.

\subsubsection*{Part (i): $k$-means (Squared loss) under local quadratic growth}

Throughout this part, $g(r)=r^2$ and
\[
\mathcal L_n(\mathcal C,\theta)=\sum_{j=1}^k\sum_{i\in C_j}\|x_i-\theta_j\|^2.
\]
Assume the separable regime $\gamma>0$ and Assumption~\ref{ass:local-qg} (Assumption 5.3 in the main text):
there exist $c_{\mathrm{qg}}>0$ and a neighborhood $U$ of $\theta^*$ such that for all $\theta\in U$,
\begin{equation}\label{eq:local_qg_recall}
\mathcal L_n(\mathcal C^*,\theta)-\mathcal L_n(\mathcal C^*,\theta^*)
\ge c_{\mathrm{qg}}\sum_{j=1}^k|C_j^*|\,\|\theta_j-\theta_j^*\|^2.
\end{equation}
Let $(\hat{\mathcal C},\hat\theta)$ be $(1+\delta)$-near-optimal:
\begin{equation}\label{eq:near_opt_recall}
\mathcal L_n(\hat{\mathcal C},\hat\theta)\le (1+\delta)\mathsf{OPT}_n,
\end{equation}
and recall the benchmark approximation error $\delta_{\mathrm{approx}}$ is defined by
\[
\mathcal L_n(\mathcal C^*,\theta^*)=(1+\delta_{\mathrm{approx}})\mathsf{OPT}_n.
\]
Assume additionally that the aligned candidate prototypes satisfy $\hat\theta\in U$.
Since $U$ is a neighborhood of $\theta^*$, we may (if needed) shrink $U$ so that $\hat\theta\in U$ also implies $\eta<\gamma/2$; this is without loss of generality for a local statement.

\noindent \emph{\underbar{Step 0} (WLOG: take the best assignment for $\hat\theta$).}
For a fixed prototype tuple $\hat\theta$, define its best-response partition
\[
\mathcal C(\hat\theta)\in \arg\min_{\mathcal C}\mathcal L_n(\mathcal C,\hat\theta).
\]
Then $\mathcal L_n(\mathcal C(\hat\theta),\hat\theta)\le \mathcal L_n(\hat{\mathcal C},\hat\theta)$.
Hence the pair $(\mathcal C(\hat\theta),\hat\theta)$ is also $(1+\delta)$-near-optimal.
Therefore, we may replace $\hat{\mathcal C}$ by $\mathcal C(\hat\theta)$ and assume from now on that
\[
\hat{\mathcal C}\in \arg\min_{\mathcal C}\mathcal L_n(\mathcal C,\hat\theta).
\]

\noindent\emph{\underbar{Step 1} (Identify the optimal assignment under small displacement).}
By Lemma~\ref{lem:small_eta_assignment}, since $\eta<\gamma/2$ we have
\[
\hat{\mathcal C}=\mathcal C(\hat\theta)=\mathcal C^*,
\qquad\text{and hence}\qquad
\mathcal L_n(\hat{\mathcal C},\hat\theta)=\mathcal L_n(\mathcal C^*,\hat\theta).
\]
Substituting into \eqref{eq:near_opt_recall} yields
\begin{equation}\label{eq:upper_LCstar_hat}
\mathcal L_n(\mathcal C^*,\hat\theta)\le (1+\delta)\mathsf{OPT}_n.
\end{equation}

\noindent\emph{\underbar{Step 2} (Upper bound the benchmark-partition excess loss).}
Subtract $\mathcal L_n(\mathcal C^*,\theta^*)=(1+\delta_{\mathrm{approx}})\mathsf{OPT}_n$ from both sides of \eqref{eq:upper_LCstar_hat}:
\[
\mathcal L_n(\mathcal C^*,\hat\theta)-\mathcal L_n(\mathcal C^*,\theta^*)
\le \bigl[(1+\delta)-(1+\delta_{\mathrm{approx}})\bigr]\mathsf{OPT}_n
=(\delta-\delta_{\mathrm{approx}})\mathsf{OPT}_n.
\]
We will use the conservative bound
\begin{equation}\label{eq:excess_upper_conservative}
\mathcal L_n(\mathcal C^*,\hat\theta)-\mathcal L_n(\mathcal C^*,\theta^*)
\le (\delta+\delta_{\mathrm{approx}})\mathsf{OPT}_n,
\end{equation}
which is always valid since $\delta_{\mathrm{approx}}\ge 0$ and the left-hand side is nonnegative by \eqref{eq:local_qg_recall}.

\noindent\emph{\underbar{Step 3} (Lower bound via local quadratic growth).}
Because $\hat\theta\in U$, the local growth inequality \eqref{eq:local_qg_recall} applies with $\theta=\hat\theta$:
\begin{equation}\label{eq:excess_lower_qg}
\mathcal L_n(\mathcal C^*,\hat\theta)-\mathcal L_n(\mathcal C^*,\theta^*)
\ge c_{\mathrm{qg}}\sum_{j=1}^k|C_j^*|\,\|\hat\theta_j-\theta_j^*\|^2.
\end{equation}

\noindent\emph{\underbar{Step 4} (Convert the weighted $\ell_2$ bound into an $\ell_\infty$ bound).}
Let $c_b:=\min_{j\in[k]}|C_j^*|/n$ be the balance coefficient and recall $\eta=\max_j\|\hat\theta_j-\theta_j^*\|$.
Then
\[
\sum_{j=1}^k|C_j^*|\,\|\hat\theta_j-\theta_j^*\|^2
\ge \min_{j}|C_j^*|\cdot \max_j\|\hat\theta_j-\theta_j^*\|^2
= (nc_b)\eta^2.
\]
Plugging this into \eqref{eq:excess_lower_qg} gives
\begin{equation}\label{eq:excess_lower_cb}
\mathcal L_n(\mathcal C^*,\hat\theta)-\mathcal L_n(\mathcal C^*,\theta^*)
\ge c_{\mathrm{qg}}\,n c_b\,\eta^2.
\end{equation}

\noindent\emph{\underbar{Step 5} (Solve for $\eta$ and express in terms of $D_{\mathrm{eff}}$).}
Combine the upper bound \eqref{eq:excess_upper_conservative} with the lower bound \eqref{eq:excess_lower_cb}:
\[
c_{\mathrm{qg}}\,n c_b\,\eta^2
\le (\delta+\delta_{\mathrm{approx}})\mathsf{OPT}_n.
\]
Divide by $n c_{\mathrm{qg}}c_b$:
\begin{equation}\label{eq:eta_in_terms_of_OPT}
\eta^2 \le \frac{\delta+\delta_{\mathrm{approx}}}{c_{\mathrm{qg}}c_b}\cdot \frac{\mathsf{OPT}_n}{n}.
\end{equation}
Finally, relate $\mathsf{OPT}_n/n$ to $D_{\mathrm{eff}}$.
By definition of $D_{\mathrm{eff}}$ (effective radius), for every $i\in C_j^*$ we have $\|x_i-\theta_j^*\|\le D_{\mathrm{eff}}$, hence
\[
\mathcal L_n(\mathcal C^*,\theta^*)=\sum_{j=1}^k\sum_{i\in C_j^*}\|x_i-\theta_j^*\|^2
\le \sum_{j=1}^k\sum_{i\in C_j^*}D_{\mathrm{eff}}^2
= n D_{\mathrm{eff}}^2.
\]
Since $\mathsf{OPT}_n\le \mathcal L_n(\mathcal C^*,\theta^*)$, we get $\mathsf{OPT}_n/n\le D_{\mathrm{eff}}^2$.
Substitute this into \eqref{eq:eta_in_terms_of_OPT}:
\[
\eta^2 \le \frac{\delta+\delta_{\mathrm{approx}}}{c_{\mathrm{qg}}c_b}\,D_{\mathrm{eff}}^2,
\qquad\text{so}\qquad
\eta \le \frac{D_{\mathrm{eff}}}{\sqrt{c_{\mathrm{qg}}c_b}}\sqrt{\delta+\delta_{\mathrm{approx}}}.
\]
This proves (31) up to universal constants (absorbed in $\lesssim$). 

\begin{remark}
[when Assumption 5.3 holds automatically]
If for each $j$ the benchmark prototype $\theta_j^*$ is the empirical mean of cluster $C_j^*$,
\[
\theta_j^*=\frac1{|C_j^*|}\sum_{i\in C_j^*}x_i,
\]
then for any $\theta_j\in\mathbb R^d$ one has the exact identity
\[
\sum_{i\in C_j^*}\|x_i-\theta_j\|^2-\sum_{i\in C_j^*}\|x_i-\theta_j^*\|^2
=|C_j^*|\,\|\theta_j-\theta_j^*\|^2,
\]
obtained by expanding $\|x_i-\theta_j\|^2=\|(x_i-\theta_j^*)-(\theta_j-\theta_j^*)\|^2$ and using
$\sum_{i\in C_j^*}(x_i-\theta_j^*)=0$.
Summing over $j$ shows Assumption~\ref{ass:local-qg} holds globally with $c_{\mathrm{qg}}=1$.
\end{remark}

\subsubsection*{Part (ii): $k$-medoids (Discrete prototype space)}

Assume $g(r)=r$ and that the feasible prototype space $\Theta$ is finite (e.g., $\Theta=S=\{x_1,\dots,x_n\}$).
Define the profiled objective (value function)
\[
V(\theta):=\min_{\mathcal C}\mathcal L_n(\mathcal C,\theta),
\qquad \theta\in \Theta^k,
\]
so that the global optimum is
\[
\mathsf{OPT}_n=\min_{\theta\in\Theta^k}V(\theta).
\]
Assume the benchmark tuple $\theta^*$ is feasible (i.e., $\theta^*\in\Theta^k$), and its benchmark partition is chosen as a best response, so that
\[
V(\theta^*)=\mathcal L_n(\mathcal C^*,\theta^*)=(1+\delta_{\mathrm{approx}})\mathsf{OPT}_n.
\]
Define the minimal positive gap in the discrete landscape around the benchmark value:
\[
\Delta_{\min}
:=\min\Bigl\{\,|V(\theta)-V(\theta^*)|:\theta\in\Theta^k,\; V(\theta)\neq V(\theta^*)\,\Bigr\}.
\]
Because $\Theta^k$ is finite, $\Delta_{\min}>0$ whenever $V(\theta^*)$ is isolated among objective values (in particular, if $V$ has a unique minimizer up to permutation).

Let $(\hat{\mathcal C},\hat\theta)$ be $(1+\delta)$-near-optimal:
$\mathcal L_n(\hat{\mathcal C},\hat\theta)\le (1+\delta)\mathsf{OPT}_n$.
Then by definition of $V$,
\[
V(\hat\theta)\le \mathcal L_n(\hat{\mathcal C},\hat\theta)\le (1+\delta)\mathsf{OPT}_n,
\]
so $0\le V(\hat\theta)-\mathsf{OPT}_n\le \delta\,\mathsf{OPT}_n$.
Also, $V(\theta^*)-\mathsf{OPT}_n=\delta_{\mathrm{approx}}\mathsf{OPT}_n$.
Therefore
\begin{align*}
|V(\hat\theta)-V(\theta^*)|
&\le |V(\hat\theta)-\mathsf{OPT}_n|+|V(\theta^*)-\mathsf{OPT}_n|\\
&= (V(\hat\theta)-\mathsf{OPT}_n)+(V(\theta^*)-\mathsf{OPT}_n)\\
&\le (\delta+\delta_{\mathrm{approx}})\mathsf{OPT}_n.
\end{align*}
If $(\delta+\delta_{\mathrm{approx}})\mathsf{OPT}_n<\Delta_{\min}$, then the above implies
$|V(\hat\theta)-V(\theta^*)|<\Delta_{\min}$.
By the definition of $\Delta_{\min}$, this forces $V(\hat\theta)=V(\theta^*)$.
Under the (standard) uniqueness assumption that the prototype tuple achieving this value is unique up to permutation,
we conclude $\hat\theta$ coincides with $\theta^*$ modulo permutation, hence the aligned displacement satisfies $\eta=0$.

\subsection{Proof of Theorem~\ref{thm:tube}: The Hamming Tube}

Let $(\hat{\mathcal C}^{(1)},\hat\theta^{(1)})$ and $(\hat{\mathcal C}^{(2)},\hat\theta^{(2)})$ be two $(1+\delta)$-near-optimal solutions.
Let $\eta_m$ denote the displacement of $\hat\theta^{(m)}$ relative to $\theta^*$ as in (13) of the main text, achieved by some optimal matching permutation $\pi_m$ (the matching may differ across $m=1,2$).
Assume $\max\{\eta_1,\eta_2\}<\gamma$, so that Theorem~3.4 in the main text applies to each solution.

\noindent \emph{\underbar{Step 1} (Triangle inequality with explicit permutation composition).}
By Lemma~\ref{lem:ham_triangle} with the pivot partition $\mathcal C^*$,
\[
d_{\mathrm{Ham}}(\hat{\mathcal C}^{(1)},\hat{\mathcal C}^{(2)})
\le d_{\mathrm{Ham}}(\hat{\mathcal C}^{(1)},\mathcal C^*)+d_{\mathrm{Ham}}(\mathcal C^*,\hat{\mathcal C}^{(2)}).
\]
Since $d_{\mathrm{Ham}}$ is symmetric, the right-hand side equals
\[
d_{\mathrm{Ham}}(\hat{\mathcal C}^{(1)},\mathcal C^*)+d_{\mathrm{Ham}}(\hat{\mathcal C}^{(2)},\mathcal C^*).
\]
By Definition~2.3 (misclassification rate) and the equivalence to \eqref{eq:ham_label_def}, we can identify
\[
p_m:=p(\hat{\mathcal C}^{(m)},\mathcal C^*)=d_{\mathrm{Ham}}(\hat{\mathcal C}^{(m)},\mathcal C^*),
\qquad m=1,2.
\]
Hence
\begin{equation}\label{eq:tube_triangle}
d_{\mathrm{Ham}}(\hat{\mathcal C}^{(1)},\hat{\mathcal C}^{(2)})\le p_1+p_2.
\end{equation}

\noindent \emph{\underbar{Step 2} (Apply the global stability theorem to each solution).}
Apply Theorem~3.4 (Global stability) from the main text to solution $m$ (valid since $\eta_m<\gamma$):
\[
p_m
\le
\frac{\mathsf{OPT}_n}{n\,\Delta_g(\gamma-\eta_m;D_{\mathrm{eff}})}(\delta+\delta_{\mathrm{approx}})
+
\frac{L_g\,\eta_m}{\Delta_g(\gamma-\eta_m;D_{\mathrm{eff}})}.
\]
Summing over $m=1,2$ and using \eqref{eq:tube_triangle} yields
\begin{align}
d_{\mathrm{Ham}}(\hat{\mathcal C}^{(1)},\hat{\mathcal C}^{(2)})
&\le
\sum_{m=1}^2
\frac{\mathsf{OPT}_n}{n\,\Delta_g(\gamma-\eta_m;D_{\mathrm{eff}})}(\delta+\delta_{\mathrm{approx}})
+
\sum_{m=1}^2\frac{L_g\,\eta_m}{\Delta_g(\gamma-\eta_m;D_{\mathrm{eff}})}.
\label{eq:tube_sum_before_unify}
\end{align}

\noindent \emph{\underbar{Step 3} (Unify denominators via monotonicity of $\Delta_g$).}
Recall $\Delta_g(\cdot;D)$ is nondecreasing in its first argument:
if $0<u\le v$, then for every $r\in[0,D]$,
\[
g(r+v)-g(r)\ge g(r+u)-g(r)
\quad\Rightarrow\quad
\inf_{0\le r\le D}\{g(r+v)-g(r)\}\ge \inf_{0\le r\le D}\{g(r+u)-g(r)\}.
\]
Thus $\Delta_g(v;D)\ge \Delta_g(u;D)$.

Let $\eta_{\max}:=\max\{\eta_1,\eta_2\}$. Then $\gamma-\eta_m\ge \gamma-\eta_{\max}$, hence
\[
\Delta_g(\gamma-\eta_m;D_{\mathrm{eff}})
\ge
\Delta_g(\gamma-\eta_{\max};D_{\mathrm{eff}}),
\qquad m=1,2.
\]
Therefore,
\[
\frac1{\Delta_g(\gamma-\eta_m;D_{\mathrm{eff}})}
\le
\frac1{\Delta_g(\gamma-\eta_{\max};D_{\mathrm{eff}})}.
\]
Use this to upper bound each term in \eqref{eq:tube_sum_before_unify} by replacing
$\Delta_g(\gamma-\eta_m;D_{\mathrm{eff}})$ with the smaller quantity
$\Delta_g(\gamma-\eta_{\max};D_{\mathrm{eff}})$ in the denominator:
\begin{align*}
d_{\mathrm{Ham}}(\hat{\mathcal C}^{(1)},\hat{\mathcal C}^{(2)})
&\le
\frac{2\,\mathsf{OPT}_n}{n\,\Delta_g(\gamma-\eta_{\max};D_{\mathrm{eff}})}(\delta+\delta_{\mathrm{approx}})
+
\frac{L_g(\eta_1+\eta_2)}{\Delta_g(\gamma-\eta_{\max};D_{\mathrm{eff}})}.
\end{align*}

\noindent \emph{\underbar{Step 4} (Condition-number form).}
Using $\mathsf{OPT}_n\le \mathcal L_n(\mathcal C^*,\theta^*)\le n\,g(D_{\mathrm{eff}})$ (since $d(x_i,\theta_{c^*(i)}^*)\le D_{\mathrm{eff}}$ and $g$ is nondecreasing),
we have $\mathsf{OPT}_n/n\le g(D_{\mathrm{eff}})$, hence
\[
\frac{2\,\mathsf{OPT}_n}{n\,\Delta_g(\gamma-\eta_{\max};D_{\mathrm{eff}})}
\le
2\,\frac{g(D_{\mathrm{eff}})}{\Delta_g(\gamma-\eta_{\max};D_{\mathrm{eff}})}
=2\,\kappa(\gamma-\eta_{\max}),
\]
where we define the (effective) condition number at the shrunken margin
\[
\kappa(\gamma-\eta_{\max})
:=\frac{g(D_{\mathrm{eff}})}{\Delta_g(\gamma-\eta_{\max};D_{\mathrm{eff}})}.
\]
Therefore,
\[
d_{\mathrm{Ham}}(\hat{\mathcal C}^{(1)},\hat{\mathcal C}^{(2)})
\;\lesssim\;
2\,\kappa(\gamma-\eta_{\max})\cdot(\delta+\delta_{\mathrm{approx}})
+
\frac{L_g(\eta_1+\eta_2)}{\Delta_g(\gamma-\eta_{\max};D_{\mathrm{eff}})}.
\]
This matches the statement of Theorem~\ref{thm:tube}.

\section{Proofs for Section 7 (heterogeneous and tracking extensions)}\label{app:extensions_proofs}

In this appendix we provide complete proofs for Corollary~\ref{cor:heterogeneous} (heterogeneous stability) and Proposition~\ref{prop:tracking} (tracking stability). Both results are obtained by re-running the proof mechanism of the global stability theorem (Theorem~\ref{thm:main_stability}) with appropriately modified constants.

\subsection{Proof of Corollary~\ref{cor:heterogeneous}: Heterogeneous stability}

We work under the heterogeneous setting introduced in Section~7, where the loss contributions may vary across points.
For clarity, we recall the two quantities used in the statement.

\emph{Heterogeneous increment and Lipschitz envelope.}
For each point $i$, let $g_i:\mathbb R_+\to\mathbb R_+$ denote its loss generator (or the pointwise loss envelope in the heterogeneous model),
and define the uniform increment at margin $\alpha$ and radius $D$ by
\[
\Delta_{g_i}(\alpha;D):=\inf_{0\le r\le D}\{g_i(r+\alpha)-g_i(r)\}.
\]
Define the \emph{lower increment envelope} and \emph{upper Lipschitz envelope} by
\[
\underline{\Delta}(\alpha;D):=\inf_{i\in[n]}\Delta_{g_i}(\alpha;D),
\qquad
\overline{L}:=\sup_{i\in[n]} L_i,
\]
where $L_i$ is a (valid) Lipschitz constant of $g_i$ on the relevant domain (as assumed in Section~7).
Likewise, let $\overline{G}(D)$ denote the upper envelope of within-cluster scale, as defined in Section~7 (e.g.\ $\overline{G}(D)=\sup_i g_i(D)$).

\emph{Setup}
Let $(\hat{\mathcal C},\hat\theta)$ be any $(1+\delta)$-near-optimal solution and let $(\mathcal C^*,\theta^*)$ be the benchmark with approximation
error $\delta_{\mathrm{approx}}$, so that
\[
\mathcal L_n(\hat{\mathcal C},\hat\theta)\le (1+\delta)\mathsf{OPT}_n,
\qquad
\mathcal L_n(\mathcal C^*,\theta^*)=(1+\delta_{\mathrm{approx}})\mathsf{OPT}_n.
\]
Let $\eta$ denote the aligned displacement between $\hat\theta$ and $\theta^*$ (as in the main text),
and let $p$ denote the misclassification fraction (label-invariant) of $\hat{\mathcal C}$ relative to $\mathcal C^*$.

\emph{\underbar{Step 1}: Upper bound the global excess loss by near-optimality.}
Exactly as in the proof of Theorem~\ref{thm:main_stability},
\begin{equation}\label{eq:appC_hetero_excess_upper}
\mathcal L_n(\hat{\mathcal C},\hat\theta)-\mathcal L_n(\mathcal C^*,\theta^*)
\le (\delta+\delta_{\mathrm{approx}})\mathsf{OPT}_n.
\end{equation}

\emph{\underbar{Step 2}: Lower bound the excess loss by misclassification penalty minus displacement leakage.}
Let $M\subset[n]$ denote the set of misclassified indices under an optimal label matching, so $|M|=np$.
Under the separable regime and displacement constraint $\eta<\gamma$ (as required by Theorem~\ref{thm:main_stability}),
the same geometric argument used in the homogeneous case implies that every misclassified point incurs
a \emph{pointwise} loss increase at least $\Delta_{g_i}(\gamma-\eta;D_{\mathrm{eff}})$ relative to its benchmark assignment
(see the pointwise step in the proof of Theorem~\ref{thm:main_stability}).
Therefore, summing over misclassified points and using the envelope $\underline{\Delta}$ yields
\[
\sum_{i\in M}\bigl(g_i(\text{wrong})-g_i(\text{right})\bigr)
\ge |M|\cdot \underline{\Delta}(\gamma-\eta;D_{\mathrm{eff}})
= n p\cdot \underline{\Delta}(\gamma-\eta;D_{\mathrm{eff}}).
\]
For correctly classified points, the only way the heterogeneous loss can decrease relative to the benchmark is through the prototype
displacement. By Lipschitz continuity of each $g_i$ and the displacement bound $\eta$, the per-point decrease is bounded by $L_i\eta$,
and hence by $\overline{L}\eta$. Summing over all $n$ points gives the conservative leakage bound
\[
\sum_{i\notin M}\bigl(g_i(\text{assigned to }\hat\theta)-g_i(\text{assigned to }\theta^*)\bigr)\ge -n\,\overline{L}\eta.
\]
Combining the two parts, we obtain the global lower bound
\begin{equation}\label{eq:appC_hetero_excess_lower}
\mathcal L_n(\hat{\mathcal C},\hat\theta)-\mathcal L_n(\mathcal C^*,\theta^*)
\ge n p\cdot \underline{\Delta}(\gamma-\eta;D_{\mathrm{eff}}) - n\,\overline{L}\eta.
\end{equation}

\emph{\underbar{Step 3}: Conclude the heterogeneous stability inequality.}
Combine \eqref{eq:appC_hetero_excess_upper} and \eqref{eq:appC_hetero_excess_lower}:
\[
n p\cdot \underline{\Delta}(\gamma-\eta;D_{\mathrm{eff}}) - n\,\overline{L}\eta
\le (\delta+\delta_{\mathrm{approx}})\mathsf{OPT}_n.
\]
Divide by $n\,\underline{\Delta}(\gamma-\eta;D_{\mathrm{eff}})$ and rearrange to obtain
\[
p
\le
\frac{\mathsf{OPT}_n}{n\,\underline{\Delta}(\gamma-\eta;D_{\mathrm{eff}})}(\delta+\delta_{\mathrm{approx}})
+
\frac{\overline{L}\eta}{\underline{\Delta}(\gamma-\eta;D_{\mathrm{eff}})},
\]
which is exactly \eqref{eq:hetero_bound}.

\emph{\underbar{Step 4}: Condition-number form.}
If additionally $\mathsf{OPT}_n/n\lesssim \overline{G}(D_{\mathrm{eff}})$, then the first term in \eqref{eq:hetero_bound} satisfies
\[
\frac{\mathsf{OPT}_n}{n\,\underline{\Delta}(\gamma-\eta;D_{\mathrm{eff}})}
\;\lesssim\;
\frac{\overline{G}(D_{\mathrm{eff}})}{\underline{\Delta}(\gamma-\eta;D_{\mathrm{eff}})}
=: \kappa_{\mathrm{het}},
\]
yielding the stated condition-number form.

\subsection{Proof of Proposition~\ref{prop:tracking}: Tracking stability}

We prove Proposition~\ref{prop:tracking} by applying the global stability bound at time $t$ and then decomposing the total displacement
into an algorithmic component (warm-start inaccuracy) and an environmental component (data drift).

\emph{Time-$t$ notation and assumptions.}
At time $t$, let $\mathcal C_t^*$ and $\theta_t^*$ denote the benchmark partition and prototypes, with effective radius $D_t$ and margin $\gamma_t$.
Let $(\hat{\mathcal C}_t,\hat\theta_t)$ be the warm-started solution produced by the tracking algorithm at time $t$, and let $p_t$ denote its
misclassification fraction relative to $\mathcal C_t^*$.
Assume the total displacement from the benchmark is small:
\[
\eta_t:=\eta(\hat\theta_t,\theta_t^*) < \gamma_t,
\qquad\text{and}\qquad
\Delta_g(\gamma_t-\eta_t;D_t)>0.
\]
Assume also that the optimization gap at time $t$ is $\delta_t$ in the sense that
\[
\mathcal L_{n,t}(\hat{\mathcal C}_t,\hat\theta_t)\le (1+\delta_t)\mathrm{OPT}_{n,t},
\]
where $\mathcal L_{n,t}$ and $\mathrm{OPT}_{n,t}$ denote the time-$t$ objective and its global optimum.
Finally, assume the displacement admits the additive decomposition
\[
\eta_t \le \eta_t^{\mathrm{alg}} + \eta_t^{\mathrm{drift}},
\]
as described in Section~7 (algorithmic error plus drift).

\emph{\underbar{Step 1}: Apply the global stability mechanism at time $t$.}
Apply Theorem~\ref{thm:main_stability} to the time-$t$ instance (with $\delta=\delta_t$, $\gamma=\gamma_t$, $D_{\mathrm{eff}}=D_t$, and $\eta=\eta_t$).
This yields
\begin{equation}\label{eq:appC_tracking_raw}
p_t
\le
\frac{\mathrm{OPT}_{n,t}}{n\,\Delta_g(\gamma_t-\eta_t;D_t)}\,\delta_t
+
\frac{L_g\,\eta_t}{\Delta_g(\gamma_t-\eta_t;D_t)}.
\end{equation}
(Here we suppress $\delta_{\mathrm{approx}}$ since the tracking statement is expressed in terms of the time-$t$ optimization gap $\delta_t$; if a benchmark approximation factor is present it can be added to $\delta_t$ in the same way as in Theorem~\ref{thm:main_stability}.)

\emph{\underbar{Step 2}: Condition-number form for the optimization term.}
Define the time-$t$ condition number (at shrunken margin $\gamma_t-\eta_t$) by
\[
\kappa_t(\eta_t):=\frac{g(D_t)}{\Delta_g(\gamma_t-\eta_t;D_t)}.
\]
Using the standard within-cluster scale control $\mathrm{OPT}_{n,t}/n\lesssim g(D_t)$ (as in the main text),
the first term in \eqref{eq:appC_tracking_raw} satisfies
\[
\frac{\mathrm{OPT}_{n,t}}{n\,\Delta_g(\gamma_t-\eta_t;D_t)}\,\delta_t
\;\lesssim\;
\kappa_t(\eta_t)\,\delta_t.
\]

\emph{\underbar{Step 3}: Decompose the displacement term.}
By the assumed decomposition $\eta_t\le \eta_t^{\mathrm{alg}}+\eta_t^{\mathrm{drift}}$,
\[
\frac{L_g\,\eta_t}{\Delta_g(\gamma_t-\eta_t;D_t)}
\le
\frac{L_g(\eta_t^{\mathrm{alg}}+\eta_t^{\mathrm{drift}})}{\Delta_g(\gamma_t-\eta_t;D_t)}.
\]

\emph{\underbar{Step 4}: Combine bounds.}
Combining the last two displays with \eqref{eq:appC_tracking_raw} yields
\[
p_t
\;\lesssim\;
\kappa_t(\eta_t)\,\delta_t
+
\frac{L_g(\eta_t^{\mathrm{alg}}+\eta_t^{\mathrm{drift}})}{\Delta_g(\gamma_t-\eta_t;D_t)},
\]
which is exactly \eqref{eq:tracking_bound}.
\bibliographystyle{imsart-number} 
\bibliography{refs}       

\end{document}